\documentclass[letterpaper]{article} 
\usepackage{aaai24}  

\usepackage[font={small}]{caption}
\usepackage[font={scriptsize}]{subcaption}
\usepackage{paralist}
\usepackage{enumitem}
\usepackage{sidecap}
\usepackage{amssymb}
\usepackage{amsmath}
\usepackage{amsthm}
\usepackage{amsfonts}
\usepackage{dsfont}
\usepackage{bm}
\usepackage{footmisc}

\usepackage{times}  
\usepackage{helvet}  
\usepackage{courier}  
\usepackage[hyphens]{url}  
\usepackage{graphicx} 
\usepackage{natbib}  
\usepackage{caption} 
\usepackage{algorithm}
\usepackage{algorithmic}

\usepackage{color}
\usepackage{xcolor}

\definecolor{navy}{rgb}{0,0,0.0}
\definecolor{airforceblue}{rgb}{0.36, 0.54, 0.66}
\definecolor{mydarkblue}{rgb}{0,0.08,0.45}

\newcommand{\bx}{{\mathbf x}}
\newcommand{\Ec}{\mathcal{E}}
\newcommand{\iX}{\mathbb{X}}
\newcommand{\iT}{\mathbb{T}}

\newcommand{\by}{\bm{y}}

\newcommand{\bW}{\bm{w}}
\newcommand{\X}{\mathbb{X}}
\newcommand{\R}{\mathbb{R}}

\newcommand{\E}{\mathbb{E}}

\newcommand{\bE}{\bm{e}}
\newcommand{\var}{\mathrm{var}} 

\newcommand{\Cr}{\mathbb{C}}

\DeclareMathOperator*{\argmax}{argmax}
\DeclareMathOperator{\argmin}{argmin}

\newcommand{\hW}{{\bW}}
\newcommand{\dW}{\Delta\bW}
\newcommand{\DW}{\Delta}
\newcommand{\SXX}{\Sigma_{\scriptscriptstyle XX}}
\newcommand{\SYX}{\Sigma_{\scriptscriptstyle YX}}

\newcommand{\hlt}[1]{{\color{airforceblue} #1}}
\newcommand{\texp}[1]{\quad\mathrm{\hlt{(#1)}}}

\newcommand{\myparagraph}[1]{\smallskip\noindent\textbf{#1}}

\newtheorem{lemma}{Lemma}

\newtheorem{defn}{Definition}

\newtheorem*{result*}{Theorem}

\usepackage{newfloat}
\usepackage{listings}
\DeclareCaptionStyle{ruled}{labelfont=normalfont,labelsep=colon,strut=off} 
\floatstyle{ruled}
\newfloat{listing}{tb}{lst}{}
\floatname{listing}{Listing}
\title{United We Stand: Using Epoch-wise Agreement of Ensembles to Combat Overfit\thanks{This work was supported by grants from the Israeli Council of Higher Education and the Gatsby Charitable Foundation.}}
\author{
    Uri Stern, Daniel Shwartz, Daphna Weinshall
}
\affiliations{

    School of Computer Science and Engineering, The Hebrew University of Jerusalem, Jerusalem 91904, Israel\\
    ustern@gmail.com, Daniel.Shwartz1@mail.huji.ac.il, daphna@mail.huji.ac.il\\
}

\begin{document}

\maketitle

\begin{abstract}

Deep neural networks have become the method of choice for solving many classification tasks, largely because they can fit very complex functions defined over raw data. The downside of such powerful learners is the danger of overfit. In this paper, we introduce a novel ensemble classifier for deep networks that effectively overcomes overfitting by combining models generated at specific intermediate epochs during training. Our method allows for the incorporation of useful knowledge obtained by the models during the overfitting phase without deterioration of the general performance, which is usually missed when early stopping is used.

To motivate this approach, we begin with the theoretical analysis of a regression model, whose prediction -- that the variance among classifiers increases when overfit occurs -- is demonstrated empirically in deep networks in common use. Guided by these results, we construct a new ensemble-based prediction method, where the prediction is determined by the class that attains the most consensual prediction throughout the training epochs. Using multiple image and text classification datasets, we show that when regular ensembles suffer from overfit, our method eliminates the harmful reduction in generalization due to overfit, and often even surpasses the performance obtained by early stopping. Our method is easy to implement and can be integrated with any training scheme and architecture, without additional prior knowledge beyond the training set. It is thus a practical and useful tool to overcome overfit. Code is available at \url{https://github.com/uristern123/United-We-Stand-Using-Epoch-wise-Agreement-of-Ensembles-to-Combat-Overfit}.

\end{abstract}
\section{Introduction}

Deep supervised learning has achieved exceptional results in various image recognition tasks in recent years. This impressive success is largely attributed to two factors -- ready access to very large annotated datasets, and powerful architectures involving a huge number of parameters, thus capable of learning very complex functions from large datasets. However, when training very large models, we face the risk of \emph{overfit} -- when models are fine-tuned to irrelevant details in the training set. It is defined here as the co-occurrence of \emph{increase} in training accuracy and \emph{decrease} in test accuracy, which implies poorer generalization as training proceeds. Not surprisingly, given the gravity of this issue, a variety of measures have been developed through the years to combat overfit, as reviewed in Section~\ref{sec:previous-work}.

One important cause of overfit in deep learning is the existence of false labels in the dataset (also termed "noisy labels"), which the network nevertheless is able to memorize. Empirically, such noisy labels are known to be memorized late by deep networks \citep{arpit2017closer}, resulting in overfit - errors that occur later in training. As deep learning requires large annotated datasets, the problem of false labels often arises in real applications, medical data being an important example \citep{karimi2020deep}, where the reliable labeling of data is expensive. Moreover, cheap alternatives such as crowd-sourcing or automatic labeling typically introduce noisy labels to the dataset \citep{nicholson2016label}. It is thus important to develop additional tools for combating overfit, especially heavy overfit caused by noisy labels, which is our goal in this paper. 

\begin{figure}[t!]
    \centering
    \begin{subfigure}[b]{0.595\linewidth}
        \includegraphics[width=\linewidth]{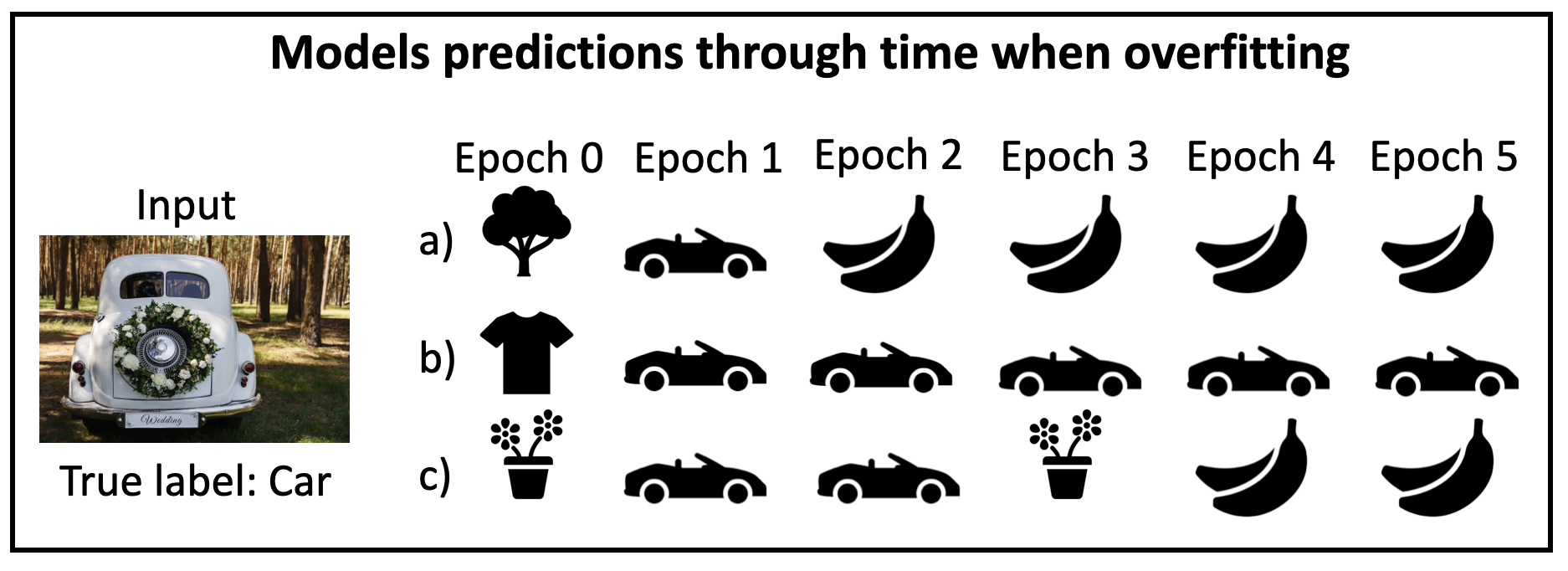}
    \vspace{-.5cm}
        \caption{Illustrative example with 3 models}
        \label{subfig:predictionvariance}
    \end{subfigure}
    \begin{subfigure}[b]{0.395\linewidth}
        \includegraphics[width=\linewidth]{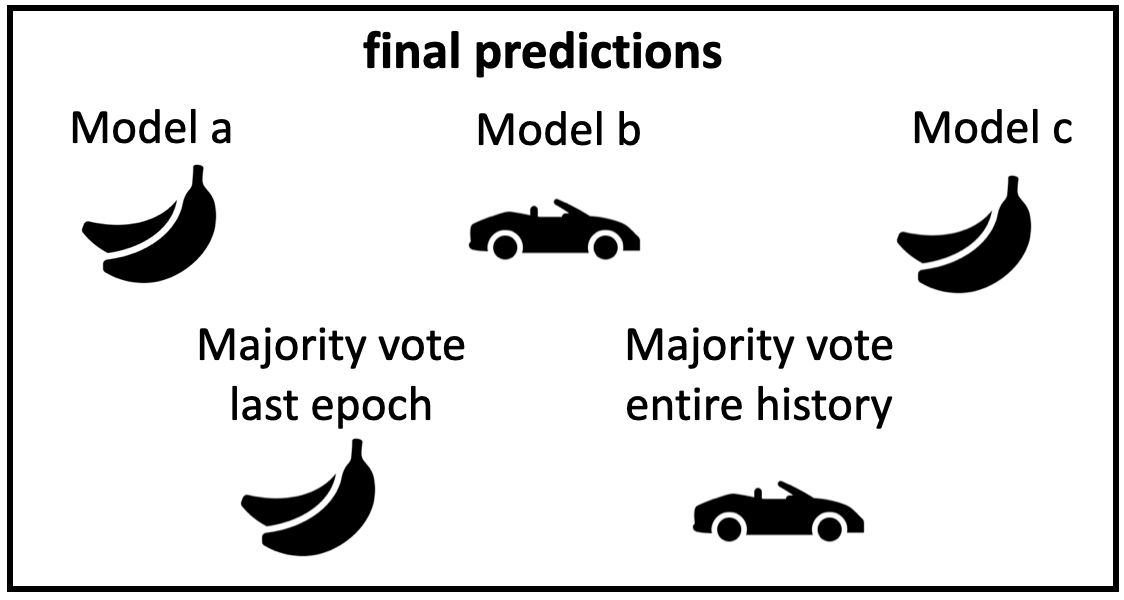}
    \vspace{-.5cm}
        \caption{Final predictions}
        \label{subfig:finalpredictions}
    \end{subfigure}
    
     \caption[motivation]{An illustrative example, showing the hypothetical behavior of an ensemble of size 3. (a) Visualization of each model's prediction over time, given as input a car image falsely labeled as a banana. (b) The final prediction of the 3 models (top row), regular ensemble (bottom left), and our method (bottom right).} 
     \label{fig:methodillustration}
     \vspace{-1.5em}
\end{figure}

Many methods have been developed to reduce overfit (See Section~\ref{sec:previous-work}), often by adding some form of regularization to the network's training. One important example is \emph{early stopping}, where the training is concluded ahead of time to avoid overfit. The problem with early stopping is that models can learn useful features from the data even as they are overfitting, which raises the need for additional methods that reduce overfit without limiting the training or the model.

Our proposed method belongs to the family of deep ensemble classifiers. Differently from most methods, our new ensemble classifier does not only consider the final predictions of the networks in the ensemble, but also tracks the networks' responses as learning proceeds. In this context, we recall recent empirical findings \citep{hacohen2020let,pliushch2021deep}, which show that when the predictions of networks over the train and test data are tracked through time, almost all the networks will either predict correctly the correct label or almost all will fail. These findings imply that deep networks demonstrate high agreement per epoch in their predictions of correct labels. But while all the networks are shown to succeed together, they may still fail in different ways, as visualized in Fig.~\ref{fig:methodillustration}. Here we pursue the hypothesis that \textbf{such diversity can be exploited to distinguish between false predictions and correct ones}, most effectively when overfit occurs.

Our work begins by investigating the hypothesis mentioned above, where we study networks' agreement over false prediction both theoretically and empirically. We start with the theoretical analysis of a regression model 
(Section~\ref{subsec:theoreticalanalysis}), which reinforces the intuition that \emph{overfit increases the prediction variance between linear models}. Accordingly, we make two conjectures: \begin{inparaenum}[(i)] \item ensembles become more effective when overfit occurs; \item the agreement between networks is reduced when overfit occurs\end{inparaenum}.  These conjectures are verified empirically in Section~\ref{subsection:empiricaloverfit} while evaluating deep networks in common use. We see that ensemble classifiers indeed improve performance when overfit occurs. However, overfit is \textbf{not} eliminated. Conjecture (ii) is validated with a \emph{new empirical result}: the variance of correct test predictions 
is much smaller than the variance of false predictions.

The aforementioned empirical result implies that correct predictions at inference time can be distinguished from false ones by looking at the \emph{variance in predictions between multiple networks throughout the training epochs}. This is used to construct a new algorithm for ensemble-based prediction, which is much more resistant to overfit (see Section~\ref{sec:alg}). This method is then tested on different image classification and text classification datasets with label noise where the networks manifest significant overfit (see Section \ref{sec:evaluation}). In these experiments we see two positive effects of our method: (i) overfit is eliminated in almost all the experiments (except in extreme cases of label noise), showing that the method is  \emph{robust against overfit}; (ii) its performance is superior to the ensemble's best epoch, when identified for early stopping. 

To conclude, our \textbf{main contribution} is a new method for ensemble-based prediction, which is resistant to overfit, and does not require any prior knowledge on the data or changes to the training protocol. It is applicable with any network and dataset, easy to implement, and is effective \emph{as a practical tool against overfit}. Most importantly, when the models are still able to learn useful patterns from the data after the occurrence of overfit, for example, when the test error shows "double descent" \citep[see][]{stephenson2021and, nakkiran2021deep, stern2023relearning}, our methods allows the user to use this forgotten knowledge while reducing overfit, outperforming the optimal early stopping. Finally, our method can be readily incorporated with existing methods, which are designed to combat overfit, in order to boost their performance, as shown in Section~\ref{sec:ablation}.

\section{Prior Work}
\label{sec:previous-work}

Previous work on combating overfit can be largely divided into two groups, which are \textbf{not} mutually exclusive. \textbf{The first group}, which includes the majority of works, focuses on modifying \emph{the learning process} in order to prevent or reduce overfit. These include various forms of data augmentation \citep{shorten2019survey} and regularization techniques, such as dropout \citep{srivastava2014dropout}, weight decay \citep{krogh1991simple}, batch normalization \citep{ioffe2015batch} and early stopping - the termination of training before overfit occurs. Such methods come at a cost, as they (and specifically early stopping) often limit the ability of the model to learn useful features from the training data in the late stages of learning, and even in the overfit phase, wherein useful features can still be learned by the model \cite{stern2023relearning}. Many methods have also been suggested for training under the presence of label noise, a known cause for overfit, see \citep{song2022learning} for a recent review. 
    
\textbf{The second group}, to which our method belongs, comes into play after learning is concluded. Such methods involve the design of post-processing algorithms to be invoked during \emph{inference time} \citep[e.g.,][]{lee2019robust}. Relevant work, when dealing with label noise, includes \citep{zhang2021learning,bae2022noisy}. An ensemble classifier is used in \citep{salman2019overfitting} to combat overfit due to insufficient training data, using network confidence to filter out erroneous predictions. Our method differs primarily in that it provides labels for all points, implicitly correcting false labels of regular ensembles, which is achieved by relying on a novel analysis of the ensemble's predictions. 
    
One benefit of the second approach is that it doesn't require additional time for training, e.g., by way of re-training or data-specific hyper-parameter tuning, which is a major drawback of the first approach. Additionally, methods from the second group can be incorporated into any training scheme, which is useful when methods from the first group fail to completely eliminate overfit. Lastly, methods from this group do not prevent the model from learning useful features at the late stages of training, as some of the first group methods do.

\textbf{Ensemble methods.}
Engaging a set of classifiers, rather than a single one, is an established methodology to deal with errors \citep{hansen1990neural}. Ensembles of deep neural networks are typically aggregated by simple yet useful methods. The two most widely used aggregation methods are the \emph{majority vote} and \emph{class probabilities average} \citep{lakshminarayanan2017simple} of the ensemble. 

Generating diversity in the ensemble can be done in a variety of ways \citep{li2018research, ganaie2021ensemble}, such as training with different datasets, as in bagging \citep{breiman1996bagging} and boosting \citep{freund1997decision}, or by varying weights and hyper-parameters \citep{wenzel2020hyperparameter}. A natural cause of diversity is learning from noisy labels, as illustrated in \citep{shwartz2022dynamic}, where ensembles of neural networks are used to identify such labels. 
As ensembles are expensive, some papers focus on minimizing their cost rather than maximizing their performance, see for example \cite{wen2020batchensemble}.

\section{Methodology}
\label{sec:methodology}

\paragraph{Notations} Let $N$ denote the number of networks in an ensemble of deep neural networks 
trained independently 
on a training dataset with $C$ classes for $E$ epochs using Stochastic Gradient Descent (SGD), and tested after each epoch of training on the test set $\iT$. Let $f_i^e(x)$ denote the prediction of network $i$ for $\bx\in\iT$ in epoch $e$. $N^e(x)$ denotes the ensemble prediction in each epoch $e$, defined as:
\begin{eqnarray}
\label{eq:highestvote}
N^e(\bx) = \argmax_{c \in C}\sum_{i=1}^{N}\mathds{1}_{[f_{i}^{e}(\bx) = c]}
\end{eqnarray}

\paragraph{Label noise generates overfit} As deep models can memorize any data distribution, having false labels in the training set (also termed noisy labels) leads networks to overfit, since their test accuracy decreases by the end of training when false labels are being memorized. We thus expose our models to data with label noise in order to evaluate robustness to heavy overfit. We evaluate our method on: \begin{inparaenum}[(i)] \item datasets with synthetic label noise, \item real-world datasets with label noise crafted via unreliable web-labeling (Clothing1M, Webvision), and \item a dataset with inherent label noise due to inherently confusing data for annotators 
(Animal10N). \end{inparaenum}
 
\paragraph{Injecting label noise} To generate data with synthetic label noise, we use  two standard noise models \citep{patrini2017making}:
\begin{enumerate}
[leftmargin=0.65cm,noitemsep]
\item \textbf{Symmetric noise:} a fraction $p \in \{0.2, 0.4, 0.6\}$ of the labels is selected randomly. Each selected label is switched to any other label with equal probability. 
\item \textbf{Asymmetric noise:} a fraction $p$ of the labels is selected randomly. Each selected label is switched to another label using a deterministic permutation function. 
\end{enumerate}

\paragraph{Inter-Model Agreement} In order to measure agreement between different models during inference time, we define below in (\ref{eq:Agreement}) an Agreement score. This score measures the average fraction of networks in the ensemble that predict class $c$ at point $\bx$ in a set of training epochs $\Ec$.
\begin{eqnarray}
\label{eq:Agreement}
Agr(\bx,c) =  \frac{1}{N\vert\Ec\vert}\sum_{e\in\Ec}\sum_{i=1}^{N}\mathds{1}_{[f_{i}^{e}(\bx) = c]}
\end{eqnarray}

\section{Overfit and Ensemble Classifiers}
\label{sec:learningdynamics}

In this section we study, theoretically and empirically, the dynamics of ensembles when overfit actually occurs (further motivation can be found in App.~\ref{subsec:motivation}). We begin with a theoretical analysis of a regression model of linear classifiers (Section~\ref{subsec:theoreticalanalysis}), showing that the agreement between such classifiers decreases when overfit occurs. In Section~\ref{subsection:empiricaloverfit} we empirically study the relevance of this result to deep learning, showing that the same phenomenon occurs also with deep neural networks. We conclude in Section~\ref{subsec:rememberhistory} with the analysis of the onset time of correct and false predictions. This empirical analysis shows that false predictions caused by overfit are associated with lower agreement scores as compared to the correct predictions, implying that false predictions are less common throughout the training than the corresponding correct predictions
. Later in Section~\ref{sec:alg}, these results guide the construction of a new ensemble-based prediction algorithm, which is resistant to overfit. 

\subsection{Overfit in linear models: theoretical result}
\label{subsec:theoreticalanalysis}

Since deep learning models are difficult to analyze theoretically, common practice invokes simple models (such as linear regression) that can be theoretically analyzed, whose analysis may shed light on the behavior of actual deep models. Accordingly, 
we formally analyze an ensemble of linear regression models trained using gradient descent via the perspective of their agreement. 
Our analysis culminates in a theorem, which states that the agreement between linear regression models decreases when overfit occurs in all the models and their individual generalization error increases.  
\begin{result*}
Assume that all models are trained using the same training set, and all suffer from overfit at the same time. Then, under reasonable assumptions, the disagreement between the models increases at the time overfit occurs.
\end{result*}
\noindent
Here is a brief sketch of the proof  (see App.~\ref{sec:overfit-app}):
\begin{enumerate}[leftmargin=0.65cm]
\item We measure \emph{Disagreement} by the empirical variance over models of the error vector at each test point, 
averaged over the test examples.
\item We prove the following (intuitive) Lemma~\ref{lem:overfit}: \emph{Overfit occurs in a model iff the gradient step of the model, which is computed from the training set, is negatively correlated with a vector unknown to the learner - the gradient step defined by the test set.}
\item We show that under certain asymptotic assumptions, the disagreement is approximately $-\rho$, where $\rho$ denotes the average correlation between each network's gradient step when computed using either the training set or the test set. It now readily follows from Lemma~\ref{lem:overfit} that if overfit occurs in all the models then the average correlation $\rho$ must be negative, and the aforementioned disagreement score increases.
\end{enumerate}

\subsection{Overfit in deep networks: empirical study}
\label{subsection:empiricaloverfit}

We discuss in the introduction known findings concerning ensembles of deep models, which imply that all networks are likely to succeed or fail together in their prediction. However, it is still possible that they fail in different ways, where each network predicts a different false label. With overfit in particular, they may favor different false labels at different epochs (see illustration in Fig.~\ref{fig:methodillustration}). 

We therefore start our empirical investigation by analyzing the variance of false predictions (or errors) in such ensembles. We use datasets injected with label noise in this study, in order to amplify the prevalence of overfit. The diversity in the ensembles has 2 sources: (i) different random initialization, and (ii) different random mini-batches within each epoch, where all networks train on the full datasets.

More specifically, we begin by training an ensemble of $N$ DNNs on noisy labeled datasets and then isolate the test points with erroneous prediction in the last epoch. Using this set of points, we compute the \emph{Error Consensus Score (ECS)}, which measures the number of networks that agree on each erroneous prediction (from 1 to $N$). In Fig.~\ref{fig:mistakes} we plot some empirical histograms of \emph{ECS}.

\begin{figure}[htbp]
    \begin{subfigure}[b]{0.32\linewidth}
    \centering
        \includegraphics[width=\linewidth]{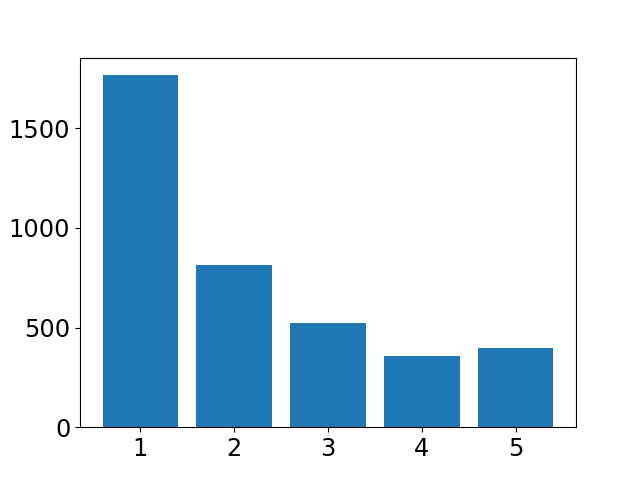}
        \caption{Cifar10 20\% AsN}
        \label{subfig:mistakesc10asym20}
    \end{subfigure}
    \begin{subfigure}[b]{0.32\linewidth}
    \centering
        \includegraphics[width=\linewidth]{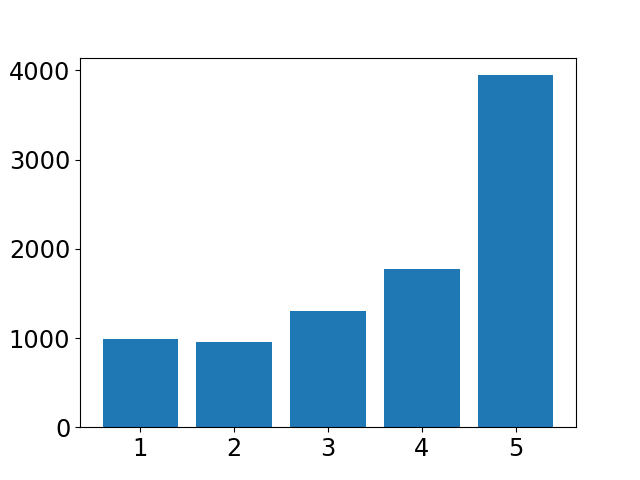}
        \caption{Cifar100 60\% SN}
        \label{subfig:mistakesc100sym60}
    \end{subfigure}
    \begin{subfigure}[b]{0.32\linewidth}
    \centering
        \includegraphics[width=\linewidth]{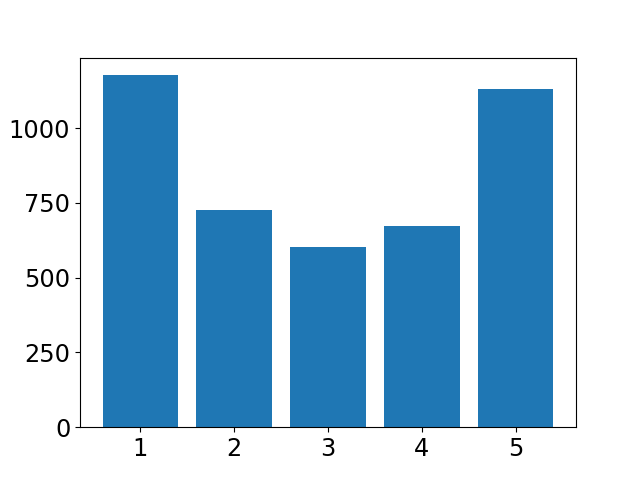}
        \caption{Cifar100 10\% AsN}
    \end{subfigure}
     \caption[mistakes]{Histograms of \emph{ECS} - the number of networks that agree on each erroneous prediction, for an ensemble of 5 networks. 'SN' denotes symmetric noise, and 'AsN' asymmetric noise. 
     } 
     \label{fig:mistakes}
     \vspace{-1.0em}
\end{figure}

The results in Fig.~\ref{fig:mistakes} indicate that when erroneous predictions are concerned, their distribution seems to have a large variance. In fact, in many of our study cases, the minority value ($\emph{ECS}<\frac{N}{2}$) dominates the distribution. We also see that in more difficult settings (e.g., Cifar100), more mistakes are made by the majority, making the ensemble unable to correct them. We therefore conjecture that ensemble classifiers will be effective in correcting erroneous predictions in cases of overfit. In addition, we conjecture that with minor overfit, ensembles will only make a little difference.

\begin{figure}[htbp]
    \begin{subfigure}[b]{0.495\linewidth}
    \centering
        \includegraphics[width=\linewidth]{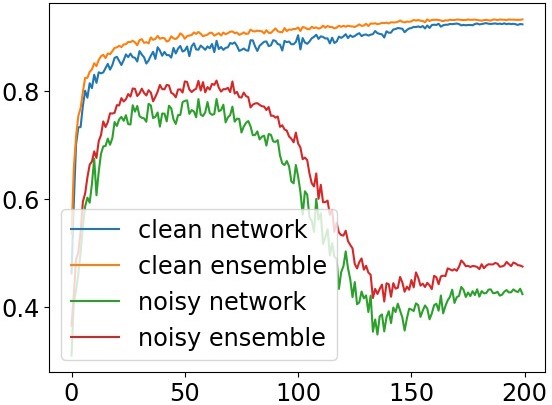}
        \caption{Cifar10, 60\% symmetric noise}
        \label{subfig:ensemblesc100sym60}
    \end{subfigure}
    \begin{subfigure}[b]{0.495\linewidth}
    \centering
        \includegraphics[width=\linewidth]{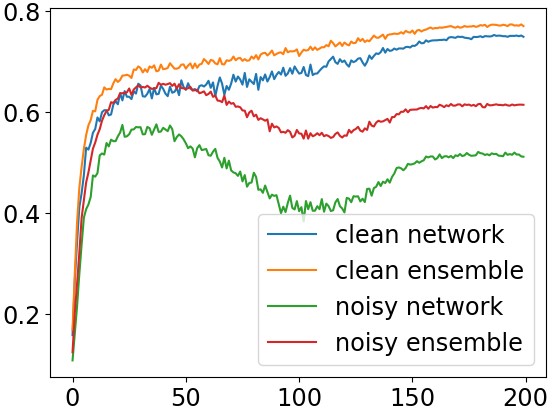}
        \caption{Cifar100, 40\% sym noise}
    \end{subfigure}
     \caption[ensembles]{Test accuracy over the epochs of a single network (green) and ensemble (red) trained on a dataset with label noise. For comparison, we also show the test accuracy of a single network (blue) and ensemble (orange) trained on the clean data only.
     } 
     \label{fig:ensembles}
     \vspace{-1.0em}
\end{figure}

These conjectures are supported by our empirical results, in experiments involving various datasets and noise settings, as shown in Fig.~\ref{fig:ensembles} and Table~\ref{table:synthesized}. In particular, Fig.~\ref{fig:ensembles} clearly shows that the benefit of the ensemble classifier is much larger when there is significant overfit, which is seen when the training data has significant label noise. However, Fig.~\ref{fig:ensembles} also shows that \textbf{while ensembles improve accuracy when there is overfit, they do not eliminate the phenomenon} - there is still performance deterioration as training proceeds. The "rebound" in test accuracy of the single network is due to the "epoch-wise double descent" phenomenon, see details in \citep{stephenson2021and, nakkiran2021deep}.

\subsection{Remember your history and avoid overfit}
\label{subsec:rememberhistory}
\begin{figure}[htbp]
    \centering
    \begin{subfigure}[b]{0.32\linewidth}
        \includegraphics[width=\linewidth]{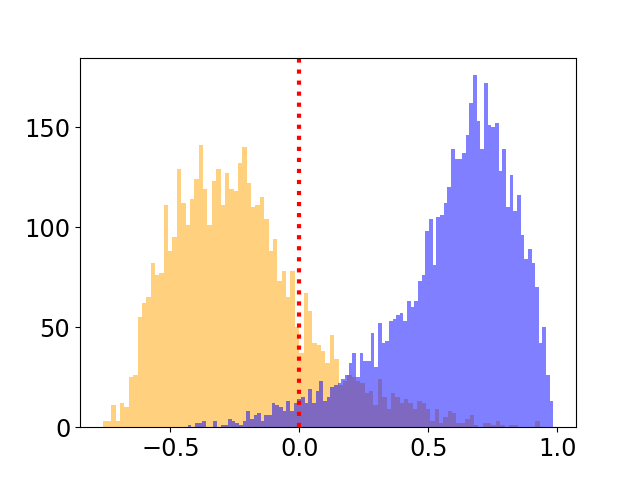}
        \caption{Cifar10, 60\% SN}
        \label{subfig:agreementmarginc10asym60}
    \end{subfigure}
    \begin{subfigure}[b]{0.32\linewidth}
        \includegraphics[width=\linewidth]{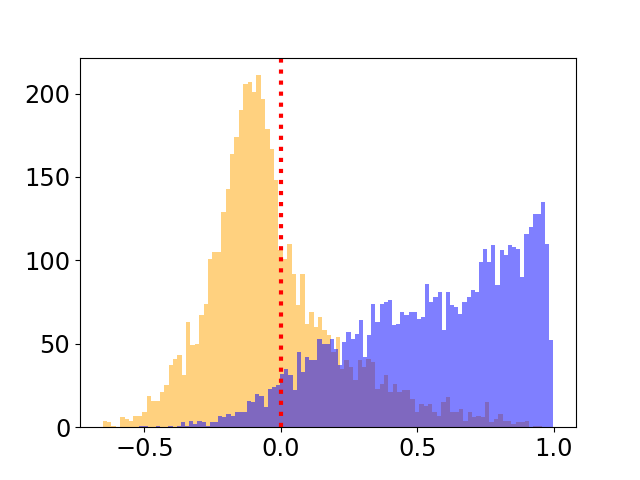}
        \caption{TinyImg, 40\% SN}
    \end{subfigure}
    \begin{subfigure}[b]{0.32\linewidth}
        \includegraphics[width=\linewidth]{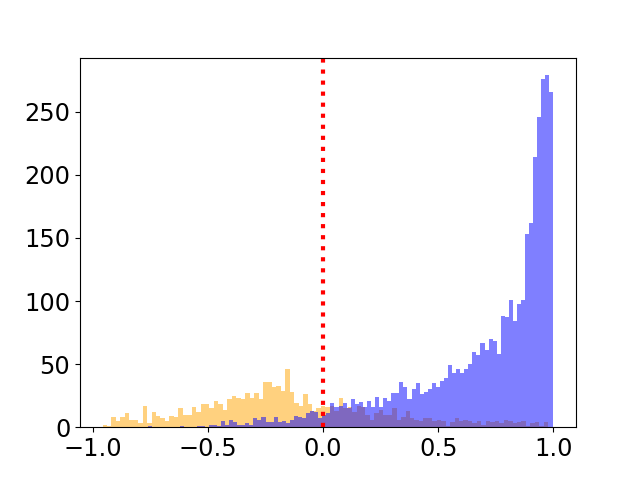}
        \caption{Imagenet100, 40\% SN}
    \end{subfigure}
     \vspace{-0.5em}
     \caption[margin histogram]{Histograms of the margin score (\ref{eq:AgreementMargin}), where blue indicates correct predictions and orange erroneous predictions.}
     \label{fig:marginhist}
\end{figure}

Ensembles are most effective when the predictions of the different members of the ensemble show a large variance, as discussed above. In this section, we inspect another potential source of variability - a prediction's persistence. 
Following our theoretical and empirical analysis and the agreement score defined in (\ref{eq:Agreement}), we hypothesize that since erroneous predictions on the test examples have large variance, correct predictions will have larger persistence (agreement) than erroneous ones, amplifying the difference between the two.

In order to test this conjecture, we use the agreement score (\ref{eq:Agreement}) defined in Section~\ref{sec:methodology}. More specifically, we compare the agreement score of the ensemble's final  prediction in epoch $E$ ($N^{E}(x)$) with the agreement score of the most agreed upon label, which is different from the final prediction: 
\begin{equation}
\begin{split}
    AgrMargin(\bx,y)&=Agr(\bx,y)-\max_{c \neq y}Agr(\bx,c) \cr
     y &= N^{E}(\bx)
\end{split}
  \label{eq:AgreementMargin}
\end{equation}
Fig.~\ref{fig:marginhist} shows histograms of score (\ref{eq:AgreementMargin}), separated to correct and erroneous predictions. Clearly, most of the predictions with negative margin scores are false, while the vast majority of the correct predictions have a positive score. The phenomenon is rather general, shown in a variety of conditions in which overfit is enhanced by injecting label noise.

\begin{figure*}[bhtp]
    \centering
    \begin{subfigure}[b]{0.245\textwidth}
        \includegraphics[width=\textwidth]{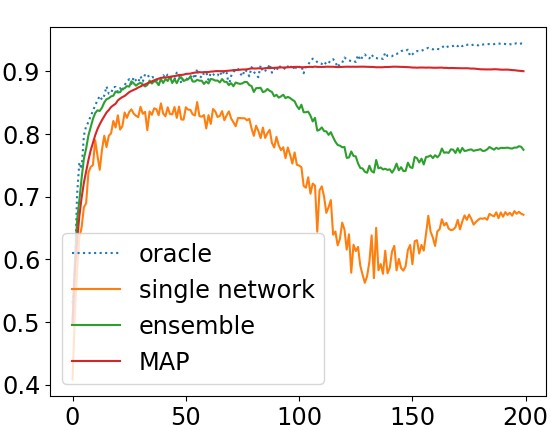}
    \vspace{-.5cm}
        \caption{Cifar10, 40\% symmetric noise}
        \label{subfig:testc10sym40}
    \end{subfigure}
    \hfill
    \begin{subfigure}[b]{0.245\textwidth}
        \includegraphics[width=\textwidth]{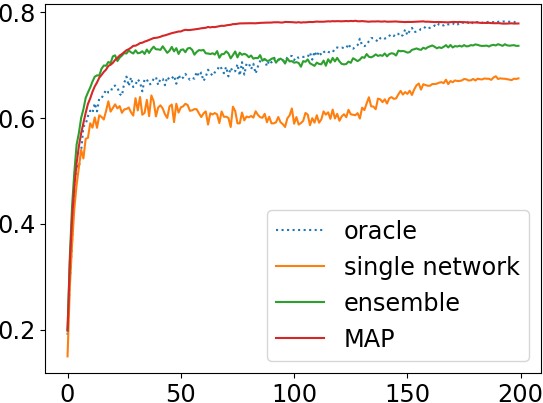}
    \vspace{-.5cm}
        \caption{Cifar100, 20\% asymmetric noise}
        \label{subfig:testc10asym20}
    \end{subfigure}
    \hfill
    \begin{subfigure}[b]{0.245\textwidth}
        \includegraphics[width=\textwidth]{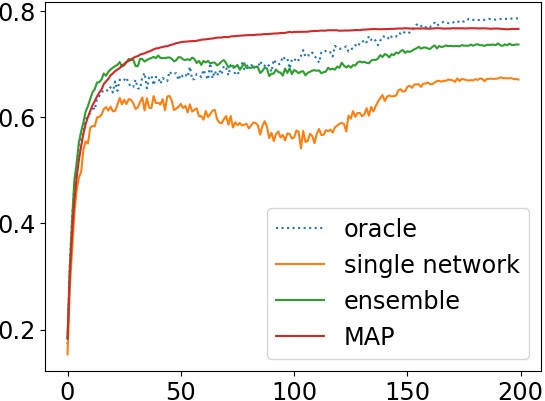}
    \vspace{-.5cm}
        \caption{Cifar100, 20\% symmetric noise}
        \label{subfig:testc100sym60}
    \end{subfigure}
    \hfill
    \begin{subfigure}[b]{0.245\textwidth}
        \includegraphics[width=\textwidth]{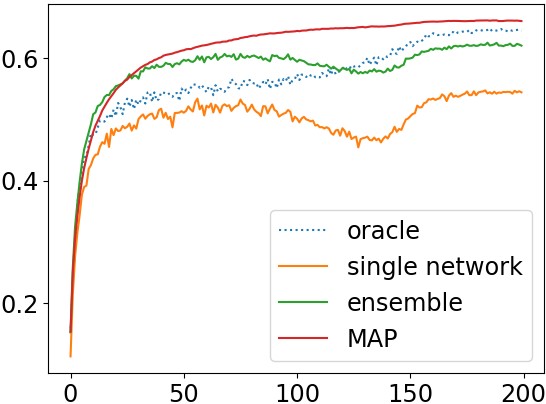}
    \vspace{-.5cm}
        \caption{TinyImg, 20\% symmetric noise}
    \end{subfigure}    
    \caption[test accuracy]{Mean test accuracy of our method (MAP, in red) in each training epoch (X-axis), compared with the following baselines: \begin{inparaenum}[(i)] \item A single network trained on the noisy dataset (orange). \item An ensemble trained on the noisy dataset (green). \item An oracle of a \emph{single network} trained on the clean subset of the data (dotted blue)\end{inparaenum}. Our method shows significant \emph{and stable} improvement, and \textbf{largely eliminates the overfit caused by the noisy labels}. Note that even though in some cases the test accuracy of the network/ensemble improves at the late stages of the training, it is often inferior compared to its past performance, and is \emph{always} inferior to our method's performance.} 
     \label{fig:testacc}
     \vspace{-.50em}
\end{figure*}

\section{Proposed Method (MAP)}
\label{sec:method}

We present a new ensemble classifier algorithm called Max Agreement Prediction (MAP) in Section~\ref{sec:alg}. Extensive experiments, demonstrating its superior performance, are described in Section~\ref{sec:results}.

\subsection{Algorithm}
\label{sec:alg}

Fig.~\ref{fig:marginhist} shows that empirically, the agreement margin score in (\ref{eq:AgreementMargin}) can be reliably used to identify false predictions. We take this result one step further and propose to use the ensemble statistics of agreement score in order to select the label prediction, instead of the usual practice of selecting
$N^{E}(\bx)$ from (\ref{eq:highestvote}). Specifically, we propose the following prediction selection rule:
\begin{equation}
    y(\bx)=\argmax_{c }Agr(\bx,c)
  \label{eq:selLabel}
\end{equation}

We use this score in Alg. \ref{alg:predcorrect} and test it in Section \ref{sec:results}, where its boost in performance is shown.
\begin{algorithm}[tb]
   \caption{max Agreement Prediction (MAP)}
\begin{algorithmic}
   \STATE {\bfseries Input:}{$f_i^e(x)$ - the prediction of network $i$ for test example $\bx\in\iT$, $~~\forall \bx,i,e\in\Ec$ 
    \STATE{\bfseries Output:}  final prediction $\forall\bx\in\iT$}
   \FOR{$\bx$ {\bfseries in} $\iT$}
   \STATE final\_prediction\_arr[x] $\gets\argmax\limits_c Agr(\bx,c)$\;
   \ENDFOR
    \STATE{Return final\_prediction\_arr};
\end{algorithmic}
\label{alg:predcorrect}
\end{algorithm}

\subsection{Empirical setup}
\label{sec:evaluation}

To evaluate our method with different levels of overfit we use image and text classification datasets with injected noisy labels (Cifar10/100, TinyImagenet, Imagenet100, MNLI, QNLI, QQP) and datasets with native label noise (Webvision50, Clothing1M, Animal10N). Training involves common methods known to reduce overfit, such as data augmentation, batch normalization and weight decay, in order to capture our method's \emph{added value} in combatting overfit. We defer full discussion of the implementation details to App.~\ref{app:implementation}.

As comparison baseline we use the following methods: \begin{inparaenum}[(i)] \item A single network of the same architecture and similarly trained. \item The same ensemble while using the `majority vote' prediction rule - $N^{E}(\bx)$ defined in (\ref{eq:highestvote}). \item The same ensemble while using the `class probabilities average' prediction rule. \end{inparaenum} The last 2 baselines are instances of a `regular ensemble', reflecting methods in common use that incur the same training cost as our method. 
A comparison with methods of comparable \emph{inference time} are discussed in Section~\ref{sec:ablation}.

Additionally, we compare our method to alternative post-processing methods, which also aim to improve classification at inference time. To assure a fair comparison, all the results are obtained using the same network architectures in the ensemble, which is most relevant to the methods that make use of an ensemble as we do \citep{salman2019overfitting, lee2019robust}. The method described in \citep{bae2022noisy} is executed several times with random seeds, the results of which are processed by the majority vote as commonly done. 

Lastly, we test our method's added value when integrated with two existing methods for training ensembles: (i) batch ensemble \citep{wen2020batchensemble}, which aims to reduce costs; and (ii) hyper ensemble \citep{wenzel2020hyperparameter}, which aims to increase the ensemble's diversity by using different hyper-parameters as well as different initialization.

\subsection{Results}
\label{sec:results}

\begin{table*}[thb!]
\footnotesize
  \centering

  \begin{tabular}{l| c|c|c|| c|c|c || c} 
    \multicolumn{1}{ c |}{Method/\textbf{Dataset}} & \multicolumn{3}{ c ||}{\textbf{Cifar10 sym}} & \multicolumn{3}{ c || }{\textbf{Cifar100 sym}} &  \multicolumn{1}{ c }{\textbf{Clothing1M}}   \\ 
    \hline
    \multicolumn{1}{ l |}{\% label noise}    & 20\% & 40\% & 60\%  &  20\% & 40\% & 60\% & 38\% (est)\\
    \hline
    \emph{single network}   & $85.4 \pm .1$& $67.7 \pm .6$& $43.5 \pm .5$& $67.0 \pm .2$ & $51.1 \pm .2$ & $32.2 \pm .4$ & $65.1 \pm .1$\\
    \emph{majority vote ensemble }   & $90.1 \pm .2$& $78.5 \pm .7$& $54.0 \pm .4$& $73.5 \pm .2$ & $61.2 \pm .4$ & $42.8 \pm .2$ & $66.0 \pm .2$\\
    \emph{probability average ensemble }   & 90.5 & 79.5& 56.6& 74.9 & 64.2 & 46.3 & 67.1\\
    \emph{MAP} & $\mathbf{93.2 \pm .1}$& $\mathbf{90.0 \pm .1}$& $\mathbf{83.8 \pm .7}$& $\mathbf{76.7 \pm .1}$ & $\mathbf{69.7 \pm .5}$& $\mathbf{60.0 \pm .4}$ &$\mathbf{71.7 \pm .1}$\\
    \hline

  \end{tabular}
  
  \vspace{0.2cm}
  
    \begin{tabular}{l| c|c|| c|c}
    \multicolumn{1}{ c |}{Method/\textbf{Dataset}} & \multicolumn{2}{ c || }{\textbf{TinyImagenet sym}} & \multicolumn{2}{ c }{\textbf{Imagenet100 sym}}\\ 
    \hline
    \multicolumn{1}{ l |}{\% label noise}    & 20\% & 40\% &  20\% & 40\%\\
    \hline
    \emph{single network}   & $54.5 \pm .7$ & $40.3 \pm .2$ & $76.4 \pm .5$&$64.9 \pm .1$\\
    \emph{majority vote ensemble }   &$62.0 \pm .1$&$50.3 \pm .4$&$82.8 \pm .2$&$75.5 \pm .5$\\
    \emph{probability average ensemble }   & 64.0& 53.2& 83.9& 77.9 \\
    \emph{MAP} & $\mathbf{66.0 \pm .1}$ &$\mathbf{60.1 \pm .1}$ &$\mathbf{83.9 \pm .1}$&$\mathbf{80.9 \pm .3}$\\
    \hline
  \end{tabular}
  
  \vspace{0.2cm}

\begin{tabular}{l| c|c|c|| c|c|c ||c}
    \multicolumn{1}{ c |}{Method/\textbf{Dataset}} & \multicolumn{3}{ c ||}{\textbf{Cifar10 asym}} & \multicolumn{3}{ c ||}{\textbf{Cifar100 asym}} &
    \multicolumn{1}{ c }{\textbf{Animal10N}}
    \\
        \hline
    \multicolumn{1}{ l |}{\% label noise}    & 10\% & 20\% & 40\%  &  10\% & 20\% & 40\% & 8\% (est) \\
    \hline
    \emph{single network}   & $90.8 \pm .1$& $83.1 \pm .2$& $59.7 \pm .3$& $74.1 \pm .2$ & $67.5 \pm .2$ & $47.5 \pm .1$& 86.0\\
    \emph{majority vote ensemble}   & $93.5 \pm .1$& $88.4 \pm .2$& $64.0 \pm .4$& $79.0 \pm .1$ & $73.7 \pm .1$ & $53.4 \pm .2$& $87.3$\\
    \emph{probability average ensemble }   & 93.4&88.6 &63.3 & 79.8 & 74.5 & 52.9 & 87.9\\
    \emph{MAP} & $\mathbf{95.2 \pm .1}$& $\mathbf{94.4 \pm .2}$& $\mathbf{85.6 \pm .2}$& $\mathbf{80.6 \pm .08}$ & $\mathbf{77.7 \pm .1}$& $\mathbf{57.5 \pm .2}$& $87.4$\\
    \hline

  \end{tabular}

\vspace{0.2cm}

\begin{tabular}{l| c|c|| c|c || c|c }
    \multicolumn{1}{ c |}{Method/\textbf{Dataset}} & \multicolumn{2}{ c ||}{\textbf{mnli}} & \multicolumn{2}{ c || }{\textbf{qnli}} & \multicolumn{2}{ c   }{\textbf{qqp}} \\
    \hline
    \multicolumn{1}{ l |}{\% label noise}    & 20\% & 40\% & 20\%  &  40\% & 20\% & 40\%\\
    \hline
    \emph{single network}   & $79.3 \pm .07$& $74.5 \pm .06$& $86.0 \pm .05$& $74.5 \pm .05$ & $85.7 \pm .04$ & $75.6 \pm .04$  \\
    \emph{majority vote ensemble }   & $81.6 \pm .2$& $76.7 \pm .2$& $87.3 \pm .1$& $74.6 \pm .4$ & $88.1 \pm .1$ & $76.3 \pm .1$ \\
    \emph{probability average ensemble }   & 82.1& 77.3& 87.5& 74.3 & 88.2 & 77.0 \\
    \emph{MAP} & $\mathbf{82.6 \pm .09}$& $\mathbf{79.3 \pm .1}$& $\mathbf{89.0 \pm .1}$& $\mathbf{82.5 \pm .08}$ & $\mathbf{88.8 \pm .05}$& $\mathbf{82.2 \pm .1}$ \\
    \hline
  \end{tabular}

  \caption{Mean test accuracy (in percent) and standard error, comparing our method (MAP) and some baselines. In the top 3 tables, we show results with 4 image datasets with injected label noise and 2 image datasets with presumed label noise marked by (est). In the bottom table, we show results with 3 text datasets with injected label noise.}

  \label{table:synthesized}
\end{table*}

\begin{table*}[thb!]
\footnotesize
  \centering
  \begin{tabular}{l| c|c || c|c || c|c || c}
    \multicolumn{1}{ c |}{Method/\textbf{Dataset}} & \multicolumn{2}{ c ||}{\textbf{Cifar10 sym}} & \multicolumn{2}{ c || }{\textbf{Cifar100 sym}} & \multicolumn{2}{ c || }{\textbf{TinyImagenet sym}} & \multicolumn{1}{ c }{\textbf{Clothing1M}}   \\ 
    \hline
    \multicolumn{1}{ l |}{\% label noise}    & 20\% & 40\% &  20\% & 40\% & 20\% & 40\% & 38\% (est)\\
    \hline
    \emph{RoG }   & $87.4$& $81.8$& $64.3$ & $55.6$ &-&- & 68.0\\
    \emph{consensus}   & $87.4 \pm .3$& $74.7 \pm .3$&  $71.0 \pm .4$ &  $41.9 \pm .1$&$23.7 \pm .2$&$18.5 \pm .3$ & 66.9\\
    \emph{NPC }   & 89.8& $77.5$&  $73.7$ & $61.5$ & 62.1&50.3 & $70.8$\\
    \emph{MAP} & $\mathbf{93.2 \pm .07}$& $\mathbf{90.0 \pm .1}$& $\mathbf{76.7 \pm .1}$ & $\mathbf{69.7 \pm .5}$& $\mathbf{66.0 \pm .1}$ &$\mathbf{60.1 \pm .06}$ &$\mathbf{71.7 \pm .2}$\\
    \hline
  \end{tabular}

  \caption{Mean test accuracy (in percent) and standard error, comparing our method (MAP) to some of the post-processing methods discussed above. We report 3 image datasets with injected label noise and a single dataset with presumed annotation errors marked by (est). Source of alternative methods results: RoG from \citep{lee2019robust}, consensus \citep{salman2019overfitting} and NPC \citep{bae2022noisy} results are recomputed using our settings and our own implementation.}

  \label{table:comparison}
\end{table*}

Fig.~\ref{fig:testacc} and Table~\ref{table:synthesized} summarize the test accuracy results. As performance drop becomes more severe (e.g., due to increased label noise), our method significantly outperforms the regular ensemble (both majority vote and class probabilities average) at the end of the training. It often outperforms optimal early stopping (i.e. the ensemble's best performance before the overfit, see Fig.~\ref{fig:testacc}).  Lastly,  Table~\ref{table:otherensembles} shows that our method is complementary to ensemble training methods that aim to increase diversity or decrease their cost. 

Importantly, note that MAP eliminates the overfit in almost all cases: when inspecting the case studies shown in Fig.~\ref{fig:testacc}, we clearly see that the test accuracy does not deteriorate when MAP is used, unlike a single network and the regular ensemble, making it a superior alternative to early stopping. Only in severe, unrealistic settings of label noise (such as 40\% asymmetric noise) do we see a deterioration in our method's performance in the late stages of the training (though MAP still outperforms the regular ensemble), but such cases are not common.

\begin{table}[htb]
\small
    \centering
    \begin{tabular}{l| c|c||c|c}
    \multicolumn{1}{ r |}{Method} & \multicolumn{2}{ c ||}{\textbf{BatchEnsemble}} & \multicolumn{2}{ c   }{\textbf{Hyper-ensemble}} 
    \\ 
    \hline
    \multicolumn{1}{ l |}{\% label noise (sym)}    & 20\% & 40\% & 20\% &  40\% \\
    \hline
    \emph{original accuracy}   &  69.5& 52.8& 72.1& 61.3 \\
    adding \emph{MAP}   &  \textbf{72.5}& \textbf{60.1}& \textbf{74.1}& \textbf{67.6}
    \vspace{.1cm}
    \end{tabular}
    \vspace{1cm}
        \begin{tabular}{l| c|c||c|c} 
    \multicolumn{1}{ r |}{Method} & \multicolumn{2}{ c || }{\textbf{BatchEnsemble}} & \multicolumn{2}{ c }{\textbf{Hyper-ensemble}}
    \\ 
    \hline
    \multicolumn{1}{ l |}{\% label noise (asym)}    &  10\% & 20\%& 10\%  &  20\% \\
    \hline
    \emph{original accuracy}   &  75.7& 67.8& 76.9& 71.2 \\
    adding \emph{MAP}   &  \textbf{77.8}& \textbf{71.1}& \textbf{77.5}& \textbf{75.3}
    \vspace{-1.0cm}
    \end{tabular}

    \caption{The added value of our method when integrated with 2 existing ensemble methods designed to either decrease cost or increase diversity, using Cifar100 with injected label noise. 
    }
    \label{table:otherensembles}
        \vspace{-.5cm}
\end{table}

\subsection{Limitations}

Our method has two main limitations: (i) the need for multiple network training, and (ii) the need to save multiple checkpoints of the network during training to be used at inference time. Using multiple GPUs can mitigate these limitations, both for parallel training and inference. Not less importantly, in Section~\ref{sec:ablation} and Table~\ref{table:otherensembles} we show that training a few networks and saving only a few checkpoints for each one (equally spaced through time), suffice for optimal or almost optimal results.

\section{Ablation Study}
\label{sec:ablation}

Our method requires the training of a few networks and the retaining of multiple checkpoints from each one, which is computationally costly. This is justified by the large improvement achieved using our method, as shown above. In order to evaluate the practical implication of those added complexities, we investigate in Section~\ref{sec:ensemble-size} the effect of the number of networks in the ensemble on the final outcome of a regular ensemble and of our method, and in Section~\ref{sec:ensemble-sample} the effect of using only a small subset of the training checkpoints. In Section~\ref{sec:nooverfit} we evaluate our method on datasets in which the size of the training set is small, which is another known cause for overfit, and without label noise. In section~\ref{sec:elr} we check our method's additive value when unique methods for learning with label noise are used.

Summarizing the full ablation results, even with \emph{a few networks} and \emph{a few checkpoints for each}, our method achieves optimal or near-optimal performance, making it practical and useful even with limited computational resources. Interestingly, our method can be combined with methods for learning with noisy labels, improving performance when such methods fail to eliminate overfit. Our method is also useful when the training set is very small. Finally, the emerging picture is robust to changes in the networks' architecture, and our method maintains its benefit over alternative deep ensemble. Additional results, concerning the effect of architecture and training choices, are described in App.~\ref{sec:ablation-other}.

\begin{figure}[htbp]
    \centering
    \begin{subfigure}[h]{0.49\linewidth}
        \includegraphics[width=.85\linewidth]{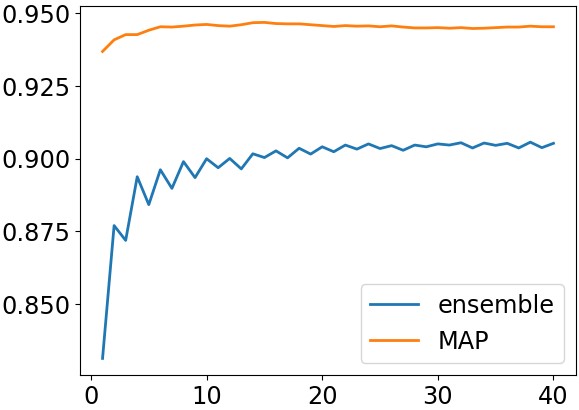}
        \caption{Cifar10, 20\% asymmetric noise}
        \label{subfig:ensemblesizecifar10}
    \end{subfigure}
    \hfill
    \begin{subfigure}[h]{0.49\linewidth}
        \includegraphics[width=.85\linewidth]{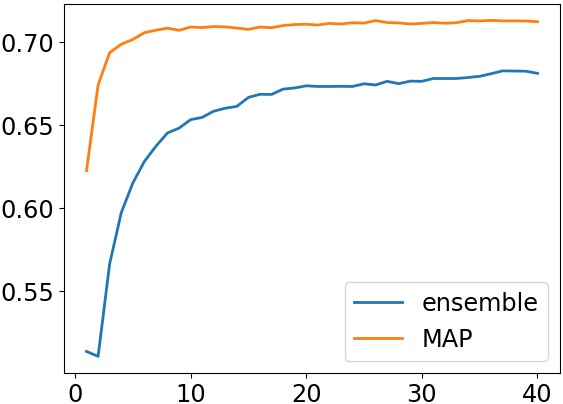}
        \caption{Cifar100, 40\% symmetric noise}
        \label{subfig:ensemblesizecifar100}
    \end{subfigure}
     \caption[ensemble size]{test accuracy of a regular ensemble and MAP when varying the size of the ensemble, where the x-axis denotes the number of networks in the ensemble.} 
     \label{fig:ensemblesize}
         \vspace{-1.em}
\end{figure}

\subsection{The effect of ensemble size}
\label{sec:ensemble-size}
In the results shown in Table~\ref{table:synthesized}, all ensemble classifiers use the same number of trained networks (5).  
However, when considering the number of networks evaluated at inference time and in order to achieve a fair comparison between regular ensembles and our method, we repeat the study without limiting the ensemble's size. Accordingly, we report accuracy at the point where the addition of models does not improve performance, see results in Table~\ref{table:hugeensembles}. Notably, our method maintains its superiority. Moreover, we see that with MAP - even a few networks achieve optimal or close to optimal results, making it practical for use (see Fig.~\ref{fig:ensemblesize}). 

\begin{table}[htb]
\small
    \centering
    \begin{tabular}{l| c|c||c|c}
    \multicolumn{1}{ c |}{Method/\textbf{Dataset}} & \multicolumn{2}{ c ||}{\textbf{Cifar10 sym}} & \multicolumn{2}{ c   }{\textbf{Cifar100 sym}} 
    \\ 
    \hline
    \multicolumn{1}{ l |}{\% label noise}    & 20\% & 40\% & 20\% &  40\% \\
    \hline
    \emph{MAP (5) }   & 93.3& 90.0& 76.6& 70.1 \\
    \emph{majority vote ($opt$)}   & 91.8& 84.2& 77.3& 69.1 \\
    \emph{MAP ($opt$)}   & 93.5& 90.7& 77.4& 71.3 
    \vspace{.1cm}
    \end{tabular}
    \vspace{1cm}
        \begin{tabular}{l| c|c||c|c}
    \multicolumn{1}{ c |}{Method/\textbf{Dataset}} & \multicolumn{2}{ c || }{\textbf{Cifar10 asym}} & \multicolumn{2}{ c }{\textbf{Cifar100 asym}}
    \\ 
    \hline
    \multicolumn{1}{ l |}{\% label noise}    &  20\% & 40\%& 20\%  &  40\% \\
    \hline
    \emph{MAP (5) }   &  94.4& 85.7& 77.9& 57.3 \\
    \emph{majority vote ($opt$)}    & 90.5& 66.5& 76.6& 56.1\\
    \emph{MAP ($opt$)}   &  94.5& 86.9& 78.4& 59.1 
    \vspace{-1.0cm}
    \end{tabular}

    \caption{Comparing ensembles with optimal size, denoted ($opt$), beyond which adding more networks does not improve the ensemble performance, to our method with 5 and $opt$ networks. }
    \label{table:hugeensembles}
        \vspace{-0.4cm}
\end{table}

\subsection{How many checkpoints are needed?}

\label{sec:ensemble-sample}

The agreement score in (\ref{eq:Agreement}) can be estimated from any subset of epochs. 
Fig.~\ref{fig:epochsize} shows test accuracy when using a subset of the training epochs, equally spaced through time, showing that even a few checkpoints are sufficient to achieve the same accuracy (or close to it) as when using all the epochs. 

\begin{figure}[htbp]
    \centering
    \begin{subfigure}[b]{0.24\linewidth}
        \includegraphics[width=\linewidth]{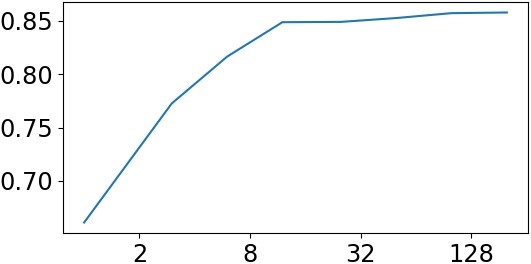}
    \vspace{-.5cm}
        \caption{\tiny Cifar10 40\% AsN}
        \label{subfig:epochsizenewc10asym40}
    \end{subfigure}
    \hfill
    \begin{subfigure}[b]{0.24\linewidth}
        \includegraphics[width=\linewidth]{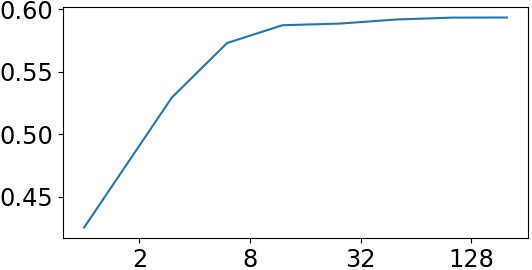}
    \vspace{-.5cm}
        \caption{\tiny Cifar100 60\% SN}
        \label{subfig:epochsizenewc100sym60}
    \end{subfigure}
    \hfill
    \begin{subfigure}[b]{0.24\linewidth}
        \includegraphics[width=\linewidth]{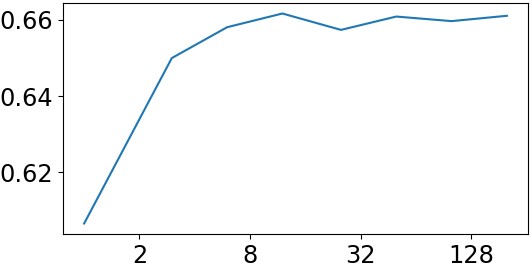}
    \vspace{-.5cm}
        \caption{\tiny TinyImg 20\% SN}
        \label{subfig:epochsizenewTimgsym20}
    \end{subfigure}
    \hfill
    \begin{subfigure}[b]{0.24\linewidth}
        \includegraphics[width=\linewidth]{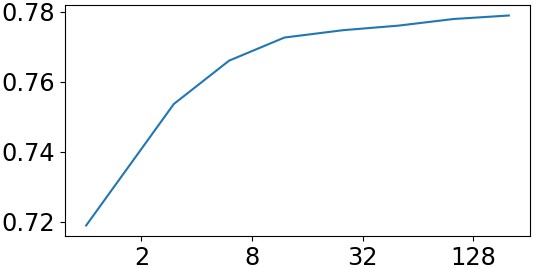}
    \vspace{-.5cm}
        \caption{\tiny Cifar100 20\% AsN}
        \label{subfig:epochsizenewc100asym20}
    \end{subfigure}
     \caption[epochs size]{Test accuracy of MAP when using only a fraction of the epochs (equally spaced) per number of epochs. 'SN' denotes symmetric noise, and 'AsN' asymmetric noise.} 
     \label{fig:epochsize}
     \vspace{-1.0em}
\end{figure}

\vspace{-.5em}

\subsection{Performance evaluation on clean datasets}
\label{sec:nooverfit}

\begin{figure}[htbp]
    \centering
    \begin{subfigure}[b]{0.32\linewidth}
        \includegraphics[width=\linewidth]{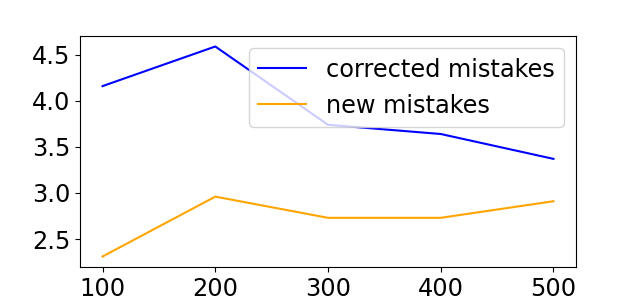}
    \vspace{-.5cm}
        \caption{TinyImagenet}
        \label{subfig:dataamountTimg}
     \vspace{-.25cm}
    \end{subfigure}
    \begin{subfigure}[b]{0.32\linewidth}
        \includegraphics[width=\linewidth]{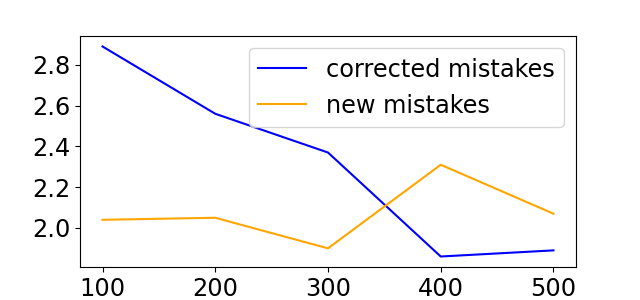}
    \vspace{-.5cm}
        \caption{Cifar100}
        \label{subfig:dataamountc100}
     \vspace{-.25cm}
   \end{subfigure}
    \begin{subfigure}[b]{0.32\linewidth}
        \includegraphics[width=\linewidth]{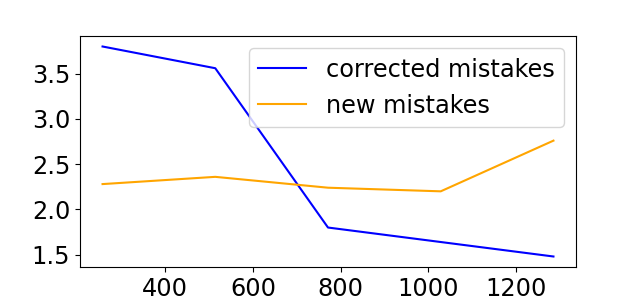}
    \vspace{-.5cm}
        \caption{Imagenet 50 classes}
        \label{subfig:dataamountImg50}
    \vspace{-.25cm}
    \end{subfigure}    
   
     \caption[data size]{X-axis: number of datapoints per class. Y-axis: fraction of predictions of a regular ensemble, which MAP changes from true to false (orange) and from false to true (blue). MAP outperforms the ensemble whenever the blue curve lies above the orange curve.}
     \label{fig:datasize}
     \vspace{-1.2em}
\end{figure}

We evaluate our method on clean datasets of different sizes, where overfit can be much smaller or nonexistent. Results are shown in Fig.~\ref{fig:datasize}. Clearly, MAP can be beneficial in such circumstances as well. %

\subsection{Combining MAP with unique method for learning with label noise}
\label{sec:elr}
In recent years, many methods have been developed to improve the model's robustness to the existence of noisy labels in the training set. As label noise is a major cause of overfit, we wanted to check whether our method can provide \emph{additional} performance gain on top of such methods, especially when they do not fully manage to eliminate the overfit. In Table~\ref{table:elr} we test our method on models trained with the ELR method \citep{liu2020early}, and show that our method can still provide significant added value to this method.

\begin{table}[h]
\small
   \centering
    \begin{tabular}{l| c|c|c}
    \multicolumn{1}{ c |}{Method/\textbf{Dataset}} & \multicolumn{3}{ c }{\textbf{TinyImagenet}} \\ 
    \hline
    \multicolumn{1}{ l |}{\% label noise (sym)}    & 20\% & 40\% & 60\% \\
    \hline
    \emph{ELR} & 57.5 & 47.8& 25.9 \\
    \emph{ELR ensemble} & 61.6 & 54.0& 33.3\\
    \emph{ELR + MAP} & $\mathbf{62.3}$ & $\mathbf{57.4}$& $\mathbf{47.2}$ \\
    \end{tabular}
    \caption{A competitive method for learning with label noise - ELR \citep{liu2020early}, and its performance with and without MAP.}
    \label{table:elr}
     \vspace{-.5em}
\end{table}

\section{Conclusions}

Overfit is a notorious problem, which afflicts deep learning, especially in the context of real-life data. 
In this paper, we propose a new ensemble classifier based on a collection of deep networks, termed MAP, which uses the evolution of predictions in each model to generate predictions that in most cases are resistant to overfit. This result is consistently shown using various datasets, including both textual and image data, in a variety of settings that exhibit heavy overfit. The method is based on a new empirical result, which shows that the agreement among deep networks decreases with the occurrence of overfit. Further support is provided by the theoretical analysis of a linear regression model.

Our empirical study focused for the most part, though not entirely, on datasets with noisy labels, a real-world condition that creates heavy overfitting, where our method is most useful. In almost all cases, the method's resistance to overfit is maintained even when the level of noise is very high. In such cases, our method significantly outperforms regular ensembles, as well as other post-processing methods designed specifically to handle noisy labels.  Its advantage over early stopping, shown in Fig.~\ref{fig:testacc}, is particularly interesting, as it shows that our method could, in some cases, allow the user to harness the new "knowledge" learned by the model at the late stages of training - even after overfit has occurred - while minimizing the damages of overfit.

Our method has some practical advantages. Notably, it is easy to implement and can be readily used at post-processing with almost any method, network, and dataset. This is accomplished \emph{without any change to the training process and without any additional hyper-parameter tuning}. Since our method is only employed at post-processing, when no overfit is suspected it can be completely avoided. 
Thus, it serves as a practical tool to combat overfit (with or without label noise in the training set), especially when large amounts of correctly labeled data are difficult to acquire.

\newpage
\bibliography{AAAIfinal}

\bigskip

\appendix

\section*{Appendix}

\section{Ensembles improve robustness to errors}
\label{subsec:motivation}

\begin{figure}[htbp]
    \centering
    \begin{subfigure}[b]{0.43\linewidth}
        \includegraphics[width=\linewidth]{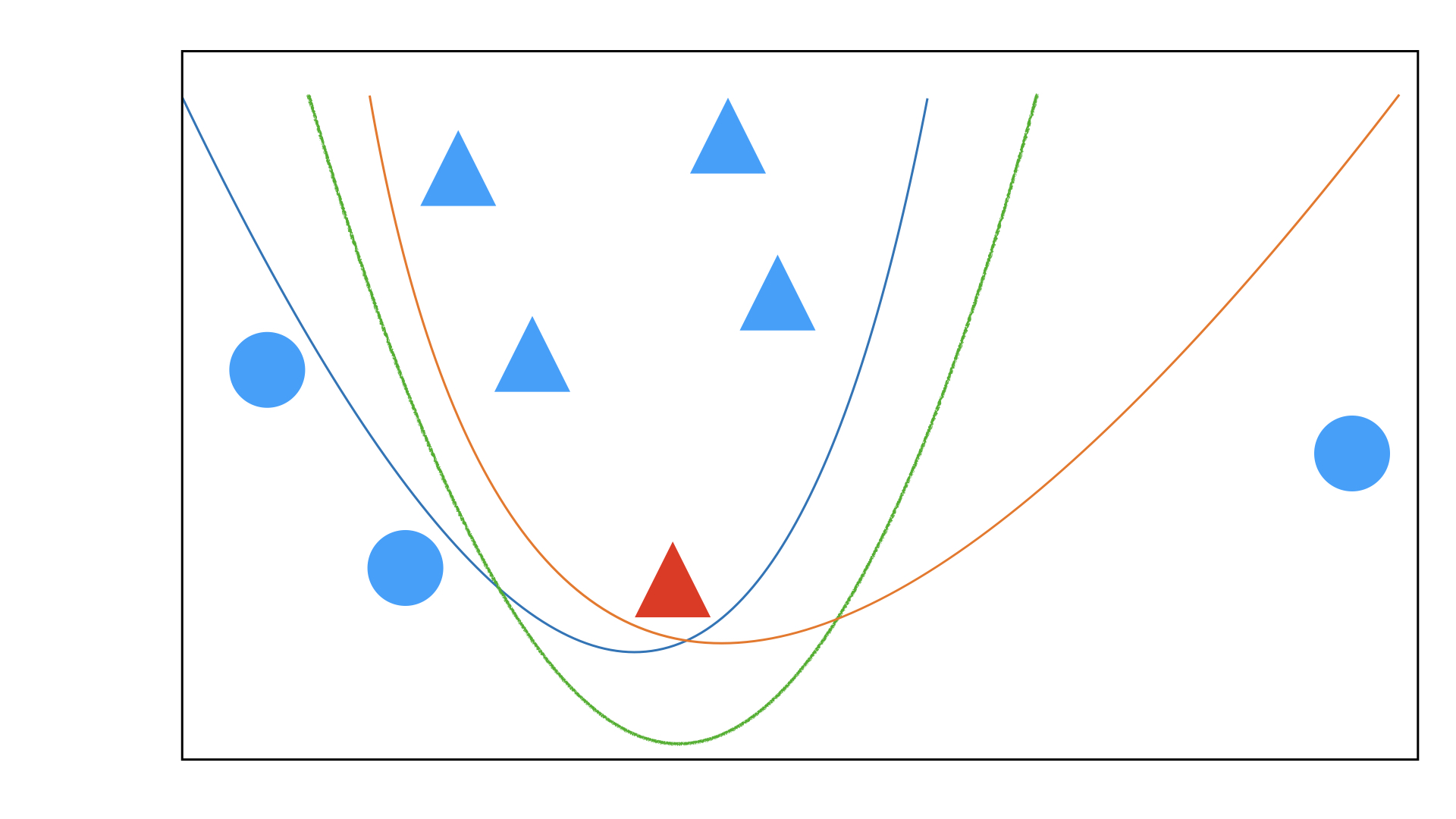}
    \vspace{-.76cm}
        \caption{3 decision boundaries}
        \label{subfig:overfittingnoisydataset}
    \vspace{-.25cm}
    \end{subfigure}
    \begin{subfigure}[b]{0.43\linewidth}
        \includegraphics[width=\linewidth]{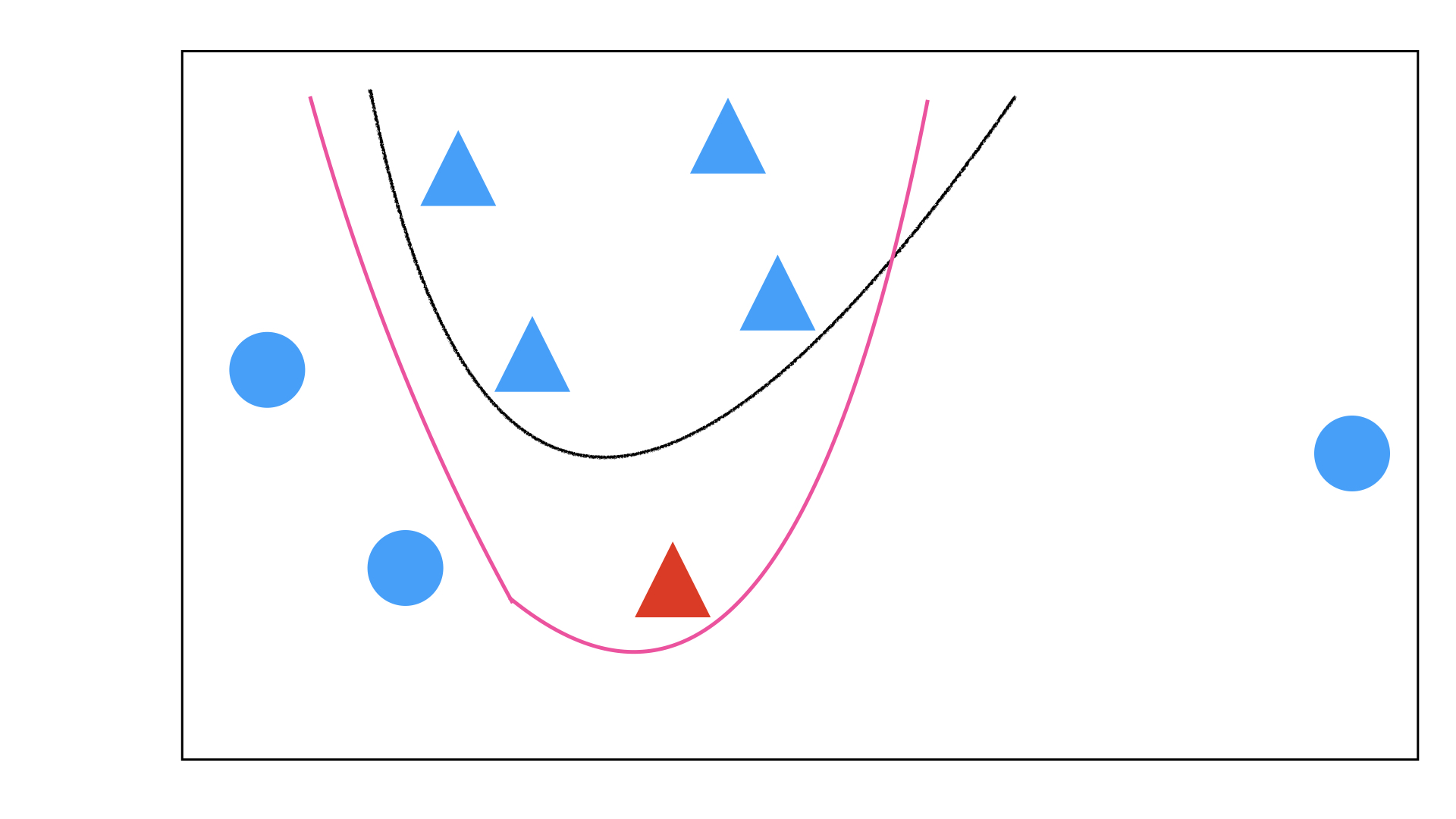}
    \vspace{-.76cm}
        \caption{Ensemble classifier}
        \label{subfig:ensembledecisionboundaries}
    \vspace{-.25cm}
    \end{subfigure}
     \caption[motivation]{(a) A $2D$ classification example, with a single false label (in red) and an ensemble of 3 models, each trained to overfit the false label. (b) The black line marks the correct decision boundary of the data's empirical distribution. The decision boundary of an ensemble classifier, based on the majority vote of the 3 instances in (a), is shown in pink. It is clearly closer to the correct decision boundary than any of the 3 members of the ensemble. } 
     \label{fig:motivationillustration}
     \vspace{-.5em}
\end{figure}

Ensemble-based methods deliver some statistics of the distribution of predictions, seeking a better estimator of the prediction of the correct classifier, and are expected to be effective in reducing error when there is variance in the classifiers' predictions. 
In Fig.~\ref{fig:motivationillustration} we demonstrate this intuition with a simple toy problem, a $2D$ example with two classes and label noise. In Fig.~\ref{subfig:overfittingnoisydataset} we show three possible decision boundaries (models), obtained since the fitting procedure is assumed to be stochastic. Because the training set has an erroneous label, the models deviate largely from the true separator. In Fig.~\ref{subfig:ensembledecisionboundaries} we plot the decision boundary of the ensemble constructed using these models (in pink), which is clearly superior, being closer to the true separator (in black).

\section{Overfit and inter-model correlation}
\label{sec:overfit-app}

In this section we formally analyze the relation between two type of scores, which  measure either overfit or inter-model agreement. \emph{Overfit} is a condition that can occur during the training of deep neural networks. It is characterized by the co-occurring decrease of train error or loss, which is continuously minimized during the training of a deep model, and the increase of test error or loss, which is the ideal measure one would have liked to minimize and which determines the network's generalization error. An \emph{agreement} score measures how similar the models are in their predictions. 

We start by introducing the model and some notations in Section~\ref{sec:notations}.  In Section~\ref{sec:overfit-agreement} we prove the main result: the occurrence of overfit at time s in all the models of the ensemble implies that the agreement
between the models decreases.

\subsection{Model and notations}
\label{sec:notations}

\myparagraph{Model.}
We analyze the agreement between an ensemble of $Q$ models, computed by solving the linear regression problem with Gradient Descent (GD) and random initialization. In this problem, the learner estimates a linear function $f(\bx): \R^d\to \R$, where $\bx\in\R^d$ denotes an input vector and $y\in\R$ the desired output. Given a training set of $M$ pairs $\{\bx_m,y_m\}_{m=1}^M$, let $X\in\R^{d\times M}$ denote the training input - a matrix whose $m^\mathrm{th}$ column is $\bx_m\in\R^d$, and let row vector $\by\in\R^M$ denote the output vector whose $m^\mathrm{th}$ element is $y_m$. When solving a linear regression problem, we seek a row vector $\hat\hW\in\R^D$ that satisfies
\begin{equation}
\label{eq:lin-problem}
\hat\hW = \argmin\limits_\hW L(\hW), \quad L(\hW) =  \frac{1}{2} \Vert \hW X-\by  \Vert_F^2
\end{equation}
To solve (\ref{eq:lin-problem}) with GD, we perform at each iterative step $s\geq 1$ the following computation:
\begin{equation}
\label{eq:gradient-step}
\begin{split}
\hW^{s+1} =& \hW^s -\mu \dW^s \\
\dW^s =& \frac{\partial L(\iX)}{\partial\hW}\bigg\vert_{\hW=\hW^s} = \hW^s\SXX-\SYX \\
&\SXX=X X^\top, ~\SYX=\by  X^\top
\end{split}
\end{equation}
for some random initialization vector $\hW_0\in\R^M$ where usually $\E[\hW_0]=0$, and learning rate $\mu$. Henceforth we omit the index $s$ when self evident from context.

As a final remark, when we use the notation $\|A\|$ below, it denotes the operator norm of the symmetric matrix $A$, namely, its largest singular value.

\myparagraph{Additional notations}
\begin{itemize}
    \item
Index $i\in[Q]$ denotes a network instance, and $t$ denotes the test data. For simplicity and with some risk of notation abuse, let $Q$ and $Q'$ also denote sets of indices, either training or test. Specifically, $Q=[1,\ldots,Q]$ and $Q'=[1,\ldots,Q,t]$. 
\item
We use \underline{function notation}, where $\{X(i),y(i)\}$ is the training set of network $i$ and $\{X(t),y(t)\}$ is the test set. Thus
\begin{equation*}
\SXX(j)=\X(j) \X(j)^\top, \SYX (j)=\by (j) \X(j)^\top \quad j\in Q'  
\end{equation*}
   \item
Similarly, $\hW (i)\in\R^d$ is the model learned by network $i$, and $\dW (i)$ is the gradient step of $\hW(i)$, where 
\begin{equation*}
\dW (i) = \hW (i) \SXX (i) -\SYX (i)  \qquad i\in Q
\end{equation*}
    \item $\bE(i,j)$ denotes a function, which maps indices $i\in Q, j\in Q'$ to the cross error of model $i$ on data $j$ - the classification error vector when using model $\hW(i)$ to estimate  $y(j)$. Let $M'=M$ if $j\in Q$ is a training index, and $M'=N$ if $j\in\{t\}$ .Then we can write
\begin{equation*}
\begin{split}
&\bE(i,j):Q\times Q'\to \R^{M'} \quad \bE(i,j) = \hW(i)\X(j)-\by (j) \\
&\implies\quad  \dW(i) = \bE(i,i) \X(i)^\top
\end{split}
\end{equation*}
Note that in this notation, $e(i,t)$ is the classification error vector when using model $i$, which is trained on data $\X(i)$, to estimate the desired outcome on the test data - $y(t)$. $\Vert e(i,t)\Vert_F$ is the test error, estimate of the generalization error, of classifier $i$.
    \item Let $\DW(i,j)$ denote the cross gradient:
\begin{equation}
\label{eq:cross}
\begin{split}
\DW(i,j) = \bE(i,j)\X(j)^\top = \hW(i)\SXX(j)-\SYX(j) \quad \\
\implies\quad  \dW(i) = \DW(i,i)
\end{split}
\end{equation}
 \end{itemize}
After each GD step, the model and the error are updated as follows:
\begin{equation*}
\begin{split}
\tilde\hW(i) &= \hW (i) - \mu\dW (i) \\ 
\tilde\bE(i,j) &= \tilde\hW(i)\X(j)-\by (j) = \bE(i,j)- \mu\DW (i,i)\X(j)
\end{split}
\end{equation*}
We note that at step $s$ and $\forall i,j\in Q$, $\tilde\hW(i)$ is a random vector in  $\R^d$, and $\tilde\bE(i,j)$ is a random vector in $\R^{M}$. If $j\in\{t\}$, then $\tilde\bE(i,j)=\tilde\bE(i,t)$ is a random vector in $\R^{N}$.

\myparagraph{Test error random variable.}
Let $N$ denote the number of test examples. Note that $\{\bE(i,t)\}_{i=1}^Q$ is a set of $Q$ test errors vectors in $\R^N$, where the $n^\mathrm{th}$ component of the $i^\mathrm{th}$ vector $\bE(i,t)_n$ captures the test error of model $i$ on test example $n$. In effect, it is a sample of size $Q$ from the random variable $\bE(*,t)_n$. This random variable captures the error over test point $n$ of a model computed from a random sample of $M$ training examples. The empirical variance of this random variable will be used to estimate the agreement between the models.

\myparagraph{Overfit.}
Overfit occurs at step $s$ if 
\begin{equation}
\label{eq:overfit-def}
\Vert \tilde\bE(i,t) \Vert_F^2 >  \Vert \bE(i,t)  \Vert_F^2
\end{equation}

\myparagraph{Measuring inter-model agreement.}
For our analysis, we need a score to measure agreement between the predictions of $Q$ linear functions. This measure is chosen to be the variance of the test error among models. Accordingly, we will measure \emph{disagreement} by the empirical variance of the test error random variable $\tilde\bE(*,t)_n$, average over all test examples $n\in[N]$.

More specifically, consider an ensemble of linear models $\{w(i)\}_{i=1}^Q$ trained on set $\iX$ to minimize (\ref{eq:lin-problem}) with $s$ gradient steps, where $i$ denotes the index of a network instance and $Q$ the number of network instances. Using the test error vectors of these models $e(i,t)$, we compute the empirical variance of each element $\var[\bE(*,t)_n]$, and sum over the test examples $n\in[N]$:
\begin{equation*}
\begin{split}
\sum_{n=1}^N\sigma^2[e(*,t)_n ] &=\sum_{n =1}^N\frac{1}{2Q^2}\sum_{i=1}^Q\sum_{j=1}^Q \vert e(i,t)_n -e(j,t)_n \vert^2 \\
&=\frac{1}{2Q^2}\sum_{i=1}^Q\sum_{j=1}^Q \Vert e(i,t)-e(j,t)\Vert^2 
\end{split}
\end{equation*}

\begin{defn}[Inter-model DisAgreement.]
The disagreement among a set of $Q$ linear models $\{w(i)\}_{i=1}^Q$ at step $s$ is defined as follows
\begin{equation}
\label{eq:agreement-s}
DisAg(s) = \frac{1}{2Q^2}\sum_{i=1}^Q\sum_{j=1}^Q \Vert e(i,t)-e(j,t)\Vert^2 
\end{equation}
\end{defn}

\subsection{Overfit and Inter-Network Agreement}
\label{sec:overfit-agreement}

We first prove Lemma~\ref{lem:overfit}, which has the following intuitive interpretation: overfit occurs in model $i$ iff the gradient step of model $i$ (denoted $\dW(i)$), which is computed using the training set, is negatively correlated with the 'correct' gradient step - the one we would have obtained had we known the test set (this unattainable vector is denoted $\DW(i,t)$).

\begin{lemma}
\label{lem:overfit}
Assume that the learning rate $\mu$ is small enough so that we can neglect terms that are $O(\mu^2)$. Then in each gradient descent step $s$, overfit occurs iff the gradient step $\dW(i) $ of network $i$ is negatively correlated with the cross gradient $\DW(i,t)$.
\end{lemma}
\begin{proof}
Starting from (\ref{eq:overfit-def})
\begin{equation*}
\begin{split}
\hspace{-5mm}\texp{overfit} \iff & \Vert \tilde\bE(i,t) \Vert_F^2 >  \Vert \bE(i,t)  \Vert_F^2  \\
\iff & \Vert \tilde\bE(i,t) \Vert_F^2 -  \Vert \bE(i,t)  \Vert_F^2 > 0 \\ \iff & \Vert \bE(i,t) - \mu\DW (i,i)X(t)\Vert_F^2 -  \Vert \bE(i,t)  \Vert_F^2 > 0\\
\iff & -2\mu\DW(i,i)X(t)\bE(i,t)^\top + O(\mu^2) > 0 \\
\iff &  \DW(i,i)\cdot\DW(i,t)  < 0  \\
\iff &  \dW(i)\cdot\DW(i,t)  < 0  
\end{split}
\end{equation*}
\end{proof}

Lemma~\ref{lem:norms} claims that if the magnitude of the gradient step $\mu$ is small enough, then the operator norm of matrix $I-\mu\SXX$ is smaller than 1. The implication is that a geometric sum of this matrix converges, a technical result which will be used later. 

\begin{lemma}
\label{lem:norms}
For any invertible covariance matrix $\SXX$ there exists $\hat\mu>0$, such that $ \mu <\hat\mu\implies\Vert I-\mu\SXX\Vert < 1$.
\end{lemma}
\begin{proof}
Since  $\SXX$ is positive-definite, we can write $\SXX=USU^\top$ for orthogonal matrix $U$ and the diagonal matrix of singular values $S=diag\{s_{i}\}$. It follows that $I-\mu\SXX=U diag\{1-\mu s_{i}\}U^\top$, a matrix whose largest singular value is $1-\mu s_d$. Since by assumption $s_d>0$, the lemma follows.
\end{proof}

Our last Lemma~\ref{lem:2} claims that eventually, after sufficiently many gradient steps, the expected value of the solution is exactly the closed-form solution of the vector that minimizes the loss.

\begin{lemma}
\label{lem:2}
Assume that $\Vert I-\mu\SXX\Vert < 1$ and $\SXX$ is invertible. If the number of gradient steps $s$ is large enough so that $\Vert I-\mu\SXX\Vert^{s}$ can be neglected, then
\begin{equation}
\label{eq:lim}
\E[\hW^{s}]  \approx \SYX \SXX^{-1} 
\end{equation}
\end{lemma}
\begin{proof}
Starting from (\ref{eq:gradient-step}), we can show that 
\begin{equation*}
\hW^{s} = \hW^0(I-\mu\SXX)^{s-1} + \mu\SYX \sum_{k=1}^{s-1} (I-\mu\SXX)^{k-1} 
\end{equation*}
Since $\E(\hW^0)=0$
\begin{equation*}
\begin{split}
    \E(\hW^{s}) &= \E(\hW^0)(I-\mu\SXX)^{s-1} + \mu\SYX \sum_{k=1}^{s-1} (I-\mu\SXX)^{k-1} \\
    &= \mu\SYX \sum_{k=1}^{s-1} (I-\mu\SXX)^{k-1} 
\end{split}
\end{equation*}
Given the lemma's assumptions, this expression can be evaluated and simplified:
\begin{equation*}
\begin{split}
\E(\hW^{s}) &=  \mu\SYX  [I-(I-\mu\SXX)]^{-1} [I-(I-\mu\SXX)^{s-1}]  \\
&= \SYX \SXX^{-1}-\SYX \SXX^{-1}(I-\mu\SXX)^{s-1}  \\ 
&\approx \SYX \SXX^{-1}  
\end{split}
\end{equation*}
\end{proof}

From (\ref{eq:agreement-s}) it follows that a decrease in inter-model agreement at step $s$, which is implied by increased test variance among models, is indicated by the following inequality:
\begin{equation}
\begin{split}
\Cr =& DisAg(s)-DisAg(s-1) \\
=& \frac{1}{2Q^2}\sum_{i,j=1}^Q  \Vert \tilde\bE(i,t)-\tilde\bE(j,t)\Vert^2  -\\
&\quad\frac{1}{2Q^2}\sum_{i,j=1}^Q \Vert \bE(i,t)-\bE(j,t)\Vert^2  ~>~0
\end{split}
\label{eq:agreement}
\end{equation}

\begin{result*}
Assume that all models see the same training set, denoted as $X(i)=X~\forall i\in[Q]$, and that the training data covariance matrix $\SXX$ is full rank.

We make the following asymptotic assumptions, which are loosely phrased but can be rigorously defined with additional notations:
\setlist{nolistsep}
\begin{enumerate}[leftmargin=0.35cm,noitemsep]
\item
The learning rate $\mu$ is small enough so that $\Vert I-\mu\SXX\Vert < 1$ (from Lemma~\ref{lem:norms}), and additionally we can neglect terms that are $O(\mu^2)$. 
\item
The number of gradient steps $s$ is large enough so that $\Vert I-\mu\SXX\Vert^{s}$ can be neglected. 
\item
The number of models $Q$ is large enough so that using the law of large numbers, we get $\frac{1}{Q}\sum_{i=1}^Q  \hW(i) \approx \E[\hW] $.
\end{enumerate}
Finally, we assume that overfit occurs at time $s$ in all the models of the ensemble. In other words, at time $s$ the generalization error does not decrease in all the models.

When these assumptions hold, the agreement between the models decreases.
\end{result*}
\begin{proof}

(\ref{eq:agreement}) can be rearranged as follows
\begin{equation*}
\begin{split}
\Cr =& \frac{1}{2Q^2}\sum_{i,j=1}^Q  \Vert [\bE(i,t)-\mu\DW (i,i)X(t)]-[\bE(j,t)-\\
&\qquad \mu\DW (j,j)X(t)]\Vert^2  - \frac{1}{2Q^2}\sum_{i,j=1}^Q \Vert \bE(i,t)-\bE(j,t)\Vert^2  \\
=& \frac{1}{Q^2} \sum_{i,j=1}^Q  -\mu [\bE(i,t)-\bE(j,t)]\cdot[\DW (i,i)X(t)]-\\
&\qquad \DW (j,j)X(t)] +O(\mu^2) \\
= &  \frac{\mu}{Q^2}\sum_{i,j=1}^Q \left [ \DW(i,i)\cdot \DW(j,t)+ \DW(j,j)\cdot \DW(i,t)\right ] -\\
&\qquad\left [ \DW(i,i)\cdot \DW(i,t)+ \DW(j,j)\cdot \DW(j,t)\right ] +O(\mu^2)\\
\end{split}
\end{equation*}
where the last transition follows from $\bE(i,t)X(t)^\top\!=\DW(i,t)$. Using assumption 2
\begin{equation}
\Cr =\mu (\Cr' - \Cr'')+O(\mu^2) \approx\mu (\Cr' - \Cr'')
\label{eq:C}
\end{equation}
where
\begin{equation}
\begin{split}
\label{eq:cr''}
\Cr'' &=\frac{1}{Q^2}\sum_{i,j=1}^Q\left [ \DW(i,i)\cdot \DW(i,t)+ \DW(j,j)\cdot \DW(j,t)\right ] \\
&= \frac{2}{Q}\sum_{i=1}^Q\DW(i,i)\cdot \DW(i,t)
\end{split}
\end{equation}
and
\begin{equation}
\label{eq:cr'}
\begin{split}
\Cr' =&\frac{1}{Q^2}\sum_{i,j=1}^Q \left [ \DW(i,i)\cdot \DW(j,t)+ \DW(j,j)\cdot \DW(i,t)\right ] \\
=&\frac{1}{Q}\sum_{i=1}^Q  \DW(i,i)\cdot \frac{1}{Q}\sum_{j=1}^Q\DW(j,t)+\\
&\qquad\frac{1}{Q}\sum_{j=1}^Q\DW(j,j)\cdot \frac{1}{Q}\sum_{i=1}^Q\DW(i,t) \\
=&\frac{1}{Q}\sum_{i=1}^Q  \DW(i,i)\cdot \frac{2}{Q}\sum_{j=1}^Q\DW(j,t)
\end{split}
\end{equation}

\begin{table*}[htb]
\small
   \centering
    \begin{tabular}{l| c||c||c||c}
    \multicolumn{1}{ c |}{Method/\textbf{Test}} & \multicolumn{1}{ c ||}{\textbf{Adversarial network}} & \multicolumn{1}{ c || }{\textbf{Transfer learning}} & \multicolumn{1}{ c || }{\textbf{Different architectures}} & \multicolumn{1}{ c }{\textbf{Different optimizer}} \\ 
    \hline
    \emph{single network}  & 52.0 & 54.9 & 52.2 & 29.4 \\
    \emph{ensemble} & 59.9 & 63.9 & 61.4 & 42.6\\
    \emph{MAP} & $\mathbf{68.7}$ & $\mathbf{70.1}$& $\mathbf{68.9}$& $\mathbf{58.5}$ \\
    \end{tabular}

    \caption{Additional ablation results. In the adversarial network setting, one network provides random predictions. In general, the single network shows the best network in the ensemble.}
    \label{table:additionalablation}
     \vspace{-.5em}
\end{table*}

Next, we prove that $\Cr'$ is approximately 0. We first deduce from assumptions 1 and 4 that
\begin{equation*}
\begin{split}
\frac{1}{Q}&\sum_{i=1}^Q  \DW(i,i) = \frac{1}{Q}\sum_{i=1}^Q  \hW (i) \SXX (i) -\SYX (i) \\
&= \left (\frac{1}{Q}\sum_{i=1}^Q  \hW (i) \right ) \SXX -\SYX \approx \E[\hW] \SXX -\SYX
\end{split}
\end{equation*}
From assumption 3 and Lemma~\ref{lem:2}, we have that $\E[\hW] \approx \SYX \SXX^{-1}$. Thus
\begin{equation*}
\begin{split}
\frac{1}{Q}\sum_{i=1}^Q  \DW(i,i) &\approx \E[\hW] \SXX -\SYX \\
&\approx \SYX \SXX^{-1}\SXX -\SYX = 0
\end{split}
\end{equation*}
From this derivation and (\ref{eq:cr'}) we may conclude that $\Cr' \approx 0$. Thus
\begin{equation}
\label{eq:c-agree}
\Cr \approx -\mu \Cr''= -\mu\frac{2}{Q}\sum_{i=1}^Q\DW(i,i)\cdot \DW(i,t)
\end{equation}

If overfit occurs at time $s$ in all the models of the ensemble, then $\Cr>0$ from Lemma~\ref{lem:overfit} and (\ref{eq:c-agree}). From (\ref{eq:agreement}) we may conclude that the inter-model agreement decreases, which concludes the proof.

\end{proof}

\stepcounter{section}

\section{Additional ablation results}
\label{sec:ablation-other}

We present in this appendix additional ablation results, demonstrating that our method is not dependent on specific architecture, optimizer and hyper-parameters, and is even robust to the presence of an adversarial network in the ensemble. 
%
Fig.~\ref{fig:ablation} shows results from experiments with different network architecture or learning rate scheduler. In all cases, our method still successfully eliminates the overfit. 

We have also tested our method on a few other settings of interest:
\begin{itemize}
    \item We tested our method on ensembles of different networks architectures/sizes trained in the same manner, and saw that the method is still effective in its performance improvement, showing comparable improvement to the one presented in Table~\ref{table:synthesized}.
    \item Transfer learning - we tested our method when initial weights are taken from pretraining (Imagenet), and saw  our method still provides large improvement in its final performance.
    \item Using different optimizer - we also tested our method when using a different optimizer (adamw) - our method still shows large improvement in performance.
    \item Having an adversarial network in the ensemble - interestingly, we also saw that due to the usage of multiple networks, our method is robust to the usage of a network that is poorly trained, or even provides random predictions, which is another advantage of our method.
\end{itemize}

The full results appear in Table~\ref{table:additionalablation}. All experiments were done with densenet121 trained on Cifar100 with 40\% symmetric noise, except transfer learning, where our network choice was resnet18 pre-trained on Imagenet.

\begin{figure}[htbp]
     \vspace{-1.0em}

    \centering
    \begin{subfigure}[b]{0.44\linewidth}
        \includegraphics[width=\linewidth]{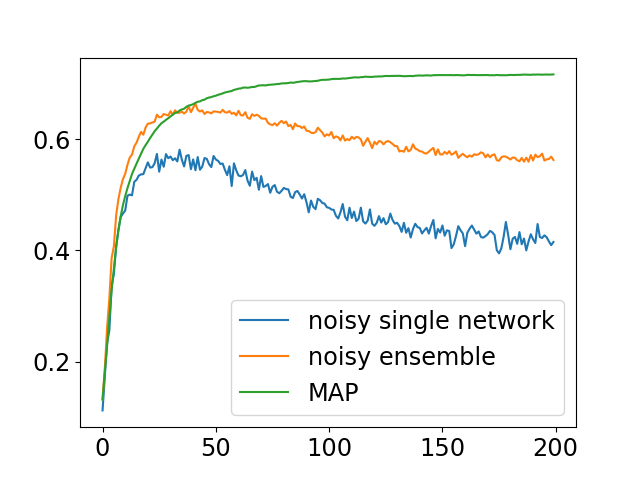}
    \vspace{-.5cm}
        \caption{Cifar100, 40\% noise, const lr}
        \label{subfig:ablationc100sym40constantLR}
    \end{subfigure}
    \begin{subfigure}[b]{0.44\linewidth}
        \includegraphics[width=\linewidth]{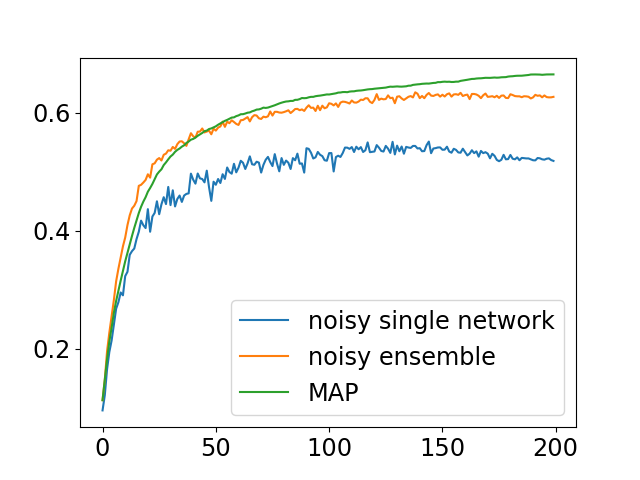}
    \vspace{-.5cm}
        \caption{Cifar100, 40\% noise, resnet32}
        \label{subfig:ablationc100sym40resnet32}
    \end{subfigure}
     \caption[ablation]{Test accuracy of MAP, different lr (a) and architecture (b).} 
     \label{fig:ablation}
     \vspace{-1.0em}
\end{figure}

\section{Implementation details}
\label{app:implementation}

\paragraph{Implementation details} We used the following datasets in our experiments: Cifar10/100 \citep{krizhevsky2009learning} with $50,000$ training images of $10/100$ classes respectively; TinyImagenet \citep{le2015tiny} with 100,000 training images of 200 classes; and Imagenet100 (cf. \citep{van2020scan}), which is a subset of 100 classes from Imagenet \citep{deng2009imagenet} with 1300 training images per class.
In these datasets, in order to amplify overfit, we used standard methods for injecting label noise (see above). As in all empirical studies and in order to evaluate the generalization error correctly, 
the test set is kept clean. We used three additional datasets, Clothing1M \citep{xiao2015learning}, Webvision \citep{li2017webvision} and Animal10N \citep{animal10n}), which are known to contain real label noise caused by automatic labeling (Clothing1M and Webvision) or human annotation mistakes (Animal10N). In Webvision, following \citep{pleiss2020identifying}, we used only the first 50 classes from Flickr and Google - around 100,000 images. For the NLP classification tasks, we used tasks from the GLUE benchmark \citep{glue}, including MNLI, QNLI and QQP. 

In our experiments the ensemble contained 5 networks and was trained for 200 epochs (except clothing, on which it was trained for 80 epochs). We used batch-size of 32, learning rate of 0.01, SGD optimizer with momentum of 0.9 and weight decay of 5e-4, cosine annealing scheduler, and standard augmentations (horizontal flip, random crops). For Cifar10/100, TinyImagenet and Imagenet100 we used DenseNet with width of 32 and batch normalization layers. For Clothing1M we used Resnet50 pretrained on Imagenet, while for Animal10N and Webvision we trained resnet50 from scratch. For elr \citep{liu2020early} we used their implementation and hyperparameters. For HyperEnsembles we used different values of learning rate, ranging from 0.01 to 0.1 at equal distances, using ensembles of 4 networks. For Batchensembles, we used 4 networks of wide-resnet, at the recommended settings. For the NLP classification tasks we used bert-base-cased \citep{devlin2018bert} with hugging face \citep{huggingface} implementation as our model, with batch size of 32, learning rate of 2e-5, max sequence length of 128, and 4 training epochs. We tested the networks every $\frac{1}{10}$ epoch. Experiments were conducted on a cluster of GPU type AmpereA10. Each experiment reports the mean and standard error (ste) results over 3 repetitions.
\vspace{-1.0em}

\end{document}